\documentclass[sigconf]{acmart}

\usepackage{multirow}
\usepackage{graphicx}
\usepackage{algorithm}
\usepackage{algpseudocode}
\usepackage{lipsum}

\usepackage{amsthm}
\newtheorem{proposition}{Proposition}

\usepackage[most]{tcolorbox}

\newtcolorbox{iterprompt}[1]{
    enhanced,
    breakable,
    colback=white,
    colframe=gray!45,
    colbacktitle=gray!60,
    coltitle=white,
    fonttitle=\bfseries\small,
    title={#1},
    boxrule=0.4pt,
    arc=2pt,
    outer arc=2pt,
    left=8pt,
    right=8pt,
    top=6pt,
    bottom=6pt,
    boxed title style={
        boxrule=0pt,
        arc=2pt
    }
}

\AtBeginDocument{%
  }

\renewcommand\footnotetextcopyrightpermission[1]{}

\begin{document}

% Provisional descriptive title; replace with your final paper title.
% \title{SegAgent: TextGrad-Guided Evaluator Agent for Upper-Body Garment Segmentation}
\title{GarmentSegAgent: Autonomous Garment Segmentation for Web Fashion Images via Textual Self-Refinement}

\author{%
\textbf{
    Jia Chen\textsuperscript{1,2,*},
    Chao Huang\textsuperscript{1,*},
    Xinrong Hu\textsuperscript{1},
    Mingdi Huang\textsuperscript{1},
    Ruyang Zhao\textsuperscript{1},
    Jun Bai\textsuperscript{3,$\dagger$}
    }
}

\affiliation{%
  \institution{%
  \textsuperscript{1} School of Computer Science and Artificial Intelligence,
  Wuhan Textile University, Wuhan, China\\
  \textsuperscript{2} Engineering Research Center of Hubei Province
  for Clothing Information, Wuhan Textile University,
  Wuhan, China\\
    \textsuperscript{3} Australian Institute for Machine Learning (AIML),
    Adelaide University, Adelaide, Australia \\
  }
  \country{\unskip}
}

\begin{abstract}

Web fashion images support applications such as visual
product search and consumer-to-shop matching, yet their
visual heterogeneity makes reliable garment segmentation
challenging. Despite advances in promptable segmentation,
autonomously identifying and correcting garment-specific
errors without human-provided spatial prompts remains
difficult. We propose \textbf{GarmentSegAgent}, an autonomous
framework that integrates semantic planning, spatial
execution, and visual evaluation through iterative
textual self-refinement. At its core, we introduce
\textbf{CAST} (Constrained Attribution-guided Selective
TextGrad), which connects visual error attribution with
targeted textual optimization. CAST identifies instruction
components responsible for segmentation errors,
selectively updates them through textual gradients,
and validates the resulting instructions against
garment-specific constraints. This enables image-specific
error correction at inference time while keeping the
underlying foundation models frozen. Experiments on
DeepFashion2, ModaNet, and Fashionpedia demonstrate
consistent improvements in region overlap and boundary
accuracy, alongside more reliable refinement than
conventional TextGrad. By enabling visual feedback
to guide controlled instruction updates, GarmentSegAgent
offers a promising direction for autonomous garment
understanding, with potential applications in scalable
fashion content analysis, visual retrieval, and
consumer-to-shop matching on the Web.

\end{abstract}

% Insert official ACM CCS concepts here before submission.


\maketitle

\pagestyle{plain}
\thispagestyle{plain}
\begingroup
\renewcommand{\thefootnote}{\fnsymbol{footnote}}
\setcounter{footnote}{1}
\footnotetext{Jia Chen and Chao Huang contributed equally to this work.}

\setcounter{footnote}{2}
\footnotetext{Corresponding author: Jun Bai
(\texttt{jun.bai@adelaide.edu.au}).}
\endgroup

% Add Introduction and Related Work here when preparing the full paper.
% The Method below belongs to the main text, not the appendix.

\section{Introduction}
\label{sec:introduction}

% Fashion images constitute an important source of visual content on the Web, supporting applications such as visual product search, fine-grained fashion retrieval, and consumer-to-shop matching~\cite{gao2020fashionbert,zhang2026beyond,ge2019deepfashion2}.
% Unlike curated product images, consumer-generated and street-fashion photographs often exhibit complex backgrounds, diverse viewpoints, severe occlusions, and overlapping garments, making reliable garment-level understanding challenging~\cite{ge2019deepfashion2,jia2020fashionpedia}.
% Recent research has increasingly emphasized localized visual representations to capture fine-grained garment attributes rather than relying solely on global image features
% ~\cite{zhang2026beyond}. In this context, accurate pixel-level garment segmentation
% provides a useful basis for isolating clothing instances and distinguishing garment appearance from surrounding visual interference. However, achieving reliable and autonomous segmentation across heterogeneous Web fashion images remains challenging, particularly when garment boundaries are ambiguous or multiple clothing instances coexist.

Fashion images are widespread across the Web
and support applications such as visual product search,
fine-grained fashion retrieval, and consumer-to-shop
matching~\cite{gao2020fashionbert,zhang2026beyond,ge2019deepfashion2}.
Unlike curated product images, consumer-generated and
street-fashion photographs often contain cluttered
backgrounds, diverse viewpoints, occlusions, and
overlapping garments, making individual clothing
instances difficult to localize and delineate
~\cite{ge2019deepfashion2,jia2020fashionpedia}.
Accurate pixel-level garment segmentation helps
isolate clothing instances and preserve fine-grained
spatial details for downstream fashion understanding.
However, reliable segmentation across such heterogeneous
imagery remains challenging, particularly when target
garments must be identified without human-provided
spatial prompts.

% Recent advances in visual segmentation have substantially improved the ability to identify and delineate objects in complex images. Early garment-specific approaches, such as FashionFormer~\cite{xu2022fashionformer}, focused on learning fine-grained
% clothing representations through supervised segmentation. More recently, foundation models such as SAM~\cite{kirillov2023sam} and SAM 2~\cite{ravi2025sam2} have enabled flexible segmentation through spatial prompts, while Grounding DINO~\cite{liu2024groundingdino} and LISA~\cite{lai2024lisa} have extended visual localization and segmentation toward language-guided understanding. Further improvements, such as SAM-REF~\cite{yu2025samref}, enhance segmentation details through image-prompt interaction.
% Despite these advances, existing approaches primarily focus on generating accurate masks from given or inferred prompts, rather than autonomously diagnosing segmentation failures
% and refining the underlying instructions. This limitation becomes particularly evident in heterogeneous Web fashion images, where ambiguous garment instances, occlusions, and overlapping clothing regions may lead to incorrect target selection or incomplete segmentation. Consequently, reliable garment extraction requires not only strong visual segmentation capabilities, but also the ability to autonomously evaluate and correct intermediate predictions.

Visual segmentation has progressed from garment-specific
architectures, such as FashionFormer~\cite{xu2022fashionformer},
to promptable foundation models ~\cite{bommasani2021foundation} including SAM~\cite{kirillov2023sam}
and SAM 2~\cite{ravi2025sam2}.
Grounding DINO~\cite{liu2024groundingdino} and
LISA~\cite{lai2024lisa} further support language-guided
object localization and segmentation, respectively,
while SAM-REF~\cite{yu2025samref} improves segmentation
details through image-prompt interaction.
Despite these advances, generating accurate masks from
spatial or textual prompts does not necessarily address
errors in the instructions guiding segmentation.
In heterogeneous Web fashion images, ambiguous targets,
occlusions, and overlapping garments can lead to incorrect
target selection or incomplete garment coverage.
Reliable autonomous segmentation therefore requires
not only effective mask generation but also the ability
to diagnose intermediate errors and refine the guiding
instructions without human intervention.

\begin{figure*}[t]
    \centering
    \includegraphics[width=\textwidth]
    {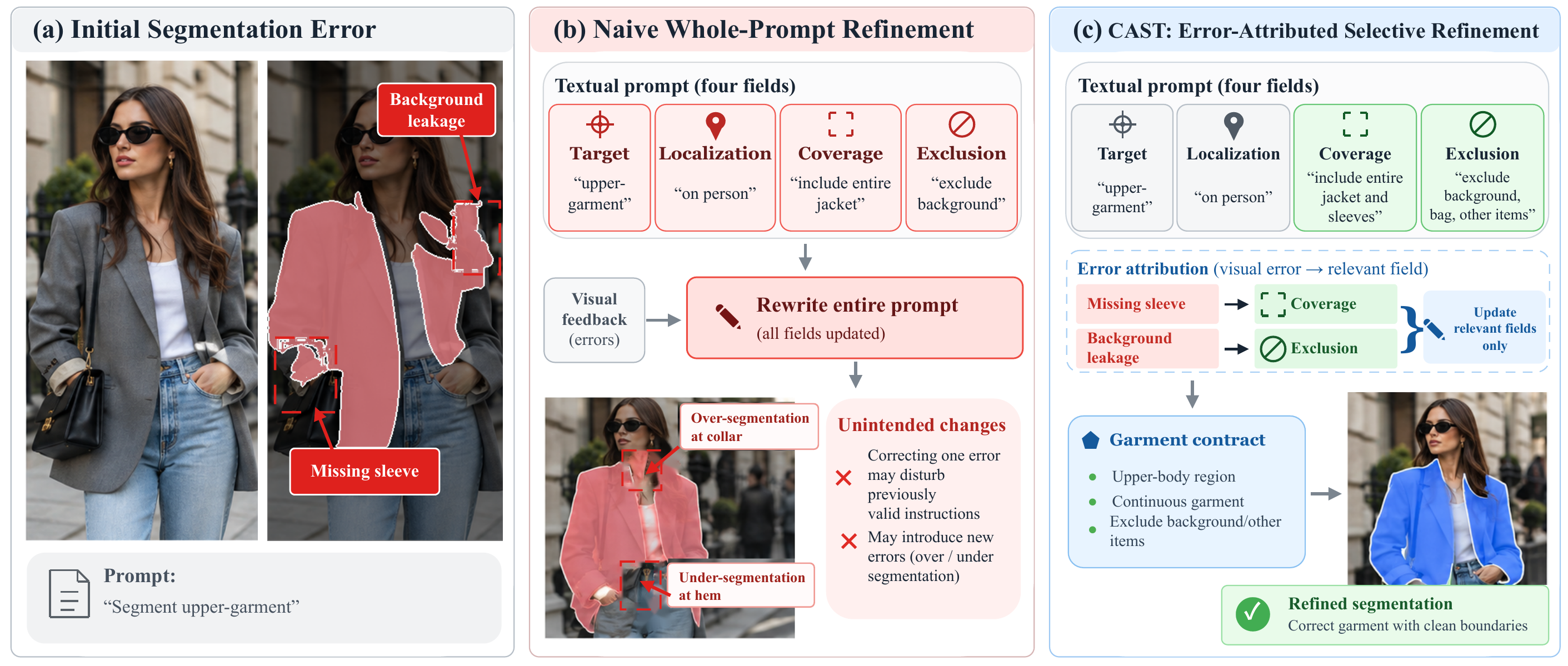}
    
\caption{Motivation for error-attributed selective textual
refinement. (a) Initial segmentation with missing garment
regions and background leakage. (b) Naive whole-prompt
refinement may disrupt otherwise correct instructions.
(c) CAST maps visual errors to relevant instruction fields
for selective, garment-constrained textual refinement.}
\captionsetup{skip=-2pt}

    \label{fig:cast_error_attributed_refinement}
    \vspace{-1em}
\end{figure*}

% To enable autonomous correction, recent research has begun to formulate segmentation as a sequential decision-making process. SegAgent~\cite{zhu2025segagent} learns segmentation actions from human annotation trajectories, while RSAgent~\cite{he2026rsagent} combines multi-turn reasoning and tool execution through supervised fine-tuning and reinforcement learning. Other efforts incorporate reflection and feedback into segmentation workflows, including Refer-Agent~\cite{jiang2026referagent} for collaborative video segmentation and guideline-consistent segmentation~\cite{vats2026guideline} through a training-free
% Worker--Supervisor framework. These developments shift the focus from producing a single mask to reasoning over successive segmentation decisions. A critical question, however, is how visual feedback should translate into effective instruction updates. For instance, a missing sleeve may require stronger coverage instructions, whereas incorrectly included body regions call for more precise exclusion constraints. Rewriting the entire instruction to address such localized errors may inadvertently alter otherwise correct segmentation decisions. This motivates a more targeted approach that associates visual errors with relevant textual instructions while preserving garment-specific constraints.

Recent studies have extended segmentation beyond single-pass
prediction by incorporating sequential reasoning and feedback.
SegAgent~\cite{zhu2025segagent} learns segmentation actions
from human annotation trajectories, while
RSAgent~\cite{he2026rsagent} integrates multi-turn reasoning
and tool execution through supervised fine-tuning and
reinforcement learning. Other approaches explore collaborative
reflection in video segmentation~\cite{jiang2026referagent}
and guideline-consistent segmentation through a training-free
Worker--Supervisor framework~\cite{vats2026guideline}.
These developments highlight the potential of agentic
segmentation, yet a key challenge remains:
how should visual errors guide updates to specific
segmentation instructions? 
% For example, a missing sleeve
% requires improved coverage guidance, whereas incorrectly
% included body regions call for stronger exclusion constraints.
% Rewriting the entire instruction may inadvertently modify
% otherwise correct target-selection or localization requirements.
% This motivates error-aware textual refinement that selectively
% updates relevant instruction components while preserving
% garment-specific constraints.
For example, the missing sleeve and background leakage in
Fig.~\ref{fig:cast_error_attributed_refinement}(a) require
distinct corrections to garment coverage and region exclusion.
This motivates connecting visual error diagnosis with
targeted instruction refinement.

% Textual optimization offers a natural approach to translating such feedback into actionable instruction updates. Early efforts, such as ProTeGi~\cite{pryzant2023automatic}, explore prompt refinement through natural-language gradients, while OPRO~\cite{yang2024opro} leverages language models to iteratively search for more effective instructions. TextGrad~\cite{textgrad} further introduces language-based backpropagation, enabling textual variables in compound AI systems to be optimized through feedback without requiring differentiable model components. For autonomous garment segmentation, however, optimization must account for the dependencies between textual instructions and their resulting spatial predictions. Visual errors may implicate different instruction components, while unrestricted updates risk disrupting previously
% correct segmentation decisions. The key challenge is therefore to determine which textual
% variables should be refined from visual evidence and how to constrain their updates to preserve the garment extraction objective.

Textual optimization provides a flexible means of refining
instructions through feedback without modifying model
parameters~\cite{bai2026fedehr, wang2026doctoragents}. ProTeGi~\cite{pryzant2023automatic} employs
natural-language gradients for prompt refinement, while
OPRO~\cite{yang2024opro} uses language models to iteratively
search for effective instructions. TextGrad~\cite{textgrad}
further enables feedback-driven optimization of textual
variables in compound AI systems through language-based
backpropagation. However, applying general-purpose textual
optimization to garment segmentation requires accounting
for the relationship between visual errors and individual
instruction components. The whole-prompt revision depicted
in Fig.~\ref{fig:cast_error_attributed_refinement}(b)
exemplifies how correcting localized errors may inadvertently
disrupt previously valid segmentation decisions.
Effective refinement therefore requires identifying
error-relevant textual variables while maintaining
garment-specific constraints.

To enable error-specific and constraint-aware refinement,
we propose \textbf{GarmentSegAgent}, an autonomous garment
segmentation framework integrating semantic planning,
spatial execution, and visual evaluation within an
iterative self-refinement process. At its core,
\textbf{CAST} (Constrained Attribution-guided Selective
TextGrad) connects visual segmentation errors with
structured textual optimization. CAST attributes observed
errors to four instruction fields---target, localization,
coverage, and exclusion---and selectively applies
TextGrad updates to the relevant variables. For instance,
CAST directs missing garment regions to coverage refinement
and background leakage to exclusion refinement, while
preserving unrelated instructions under garment-specific
constraints (Fig.~\ref{fig:cast_error_attributed_refinement}(c)).
The refined instructions subsequently guide spatial
prompt generation and segmentation, enabling autonomous,
image-specific refinement without human-provided spatial
prompts or fine-tuning the underlying foundation models.

% \begin{figure*}[t]
%     \centering
%     \includegraphics[width=\textwidth]{pictures/cast_error_attributed_refinement.drawio.pdf}
%     \caption{Illustration of error-attributed selective textual refinement. (a) An initial segmentation contains background leakage and a missing sleeve. (b) Naive whole-prompt refinement updates all instruction fields and may introduce unintended changes. (c) CAST routes each visual error to its relevant textual field, preserves the garment contract, and produces a refined segmentation.}
%     \label{fig:cast_error_attributed_refinement}
% \end{figure*}

% Our main contributions are summarized as follows:
% \begin{itemize}
%     \item We present \textbf{GarmentSegAgent}, an autonomous framework for garment segmentation in heterogeneous Web fashion images. It integrates semantic planning, spatial
%     execution, and visual evaluation into a self-refining workflow without requiring human-provided spatial prompts or foundation-model fine-tuning.

%     \item We develop \textbf{CAST}, a constrained attribution-guided selective textual optimization  mechanism. CAST associates visual segmentation errors with relevant instruction variables, performs selective TextGrad-based updates, and enforces garment-specific constraints to preserve task consistency during refinement.

%     \item We evaluate GarmentSegAgent on multiple fashion  segmentation benchmarks, including DeepFashion2, ModaNet, and Fashionpedia, with comprehensive analyses of
%     segmentation effectiveness, component contributions,  and iterative refinement behavior.

% \end{itemize}

Our main contributions are threefold.
\textbf{(1) Autonomous Garment Segmentation.}
We introduce \textbf{GarmentSegAgent}, a feedback-driven
framework that enables autonomous, iterative garment
segmentation without human-provided spatial prompts
or foundation-model fine-tuning.
\textbf{(2) Error-Attributed Textual Refinement.}
We propose \textbf{CAST}, a task-aware optimization
mechanism that maps visual segmentation errors to
relevant textual variables and performs selective
TextGrad updates under garment-specific constraints.
\textbf{(3) Comprehensive Experimental Validation.}
Experiments across three fashion benchmarks demonstrate
superior region and boundary accuracy, while refinement
analyses and ablations establish the effectiveness
of CAST in improving optimization reliability.

\section{Related Work}
\label{sec:related_work}

\paragraph{\textbf{Agentic Segmentation}}
\label{sec:agentic_segmentation}

Agentic segmentation extends mask prediction through
sequential reasoning, tool interaction, and feedback-driven correction~\cite{yao2023react,wang2024agentsurvey}. SegAgent~\cite{zhu2025segagent} learns
multi-step segmentation policies from human annotation
trajectories, whereas RSAgent~\cite{he2026rsagent}
combines multi-turn tool use with supervised fine-tuning
and reinforcement learning. Other approaches emphasize
agent coordination and iterative evaluation.
Refer-Agent~\cite{jiang2026referagent} introduces
collaborative reasoning and reflection for referring
video segmentation, while guideline-consistent
segmentation~\cite{vats2026guideline} employs a
Worker--Supervisor framework with guideline-based
feedback and adaptive termination. These approaches
advance segmentation through learned action policies,
collaborative reasoning, and iterative mask correction.

\paragraph{\textbf{Textual Optimization}}
\label{sec:textual_optimization}

% Textual optimization improves language-based instructions
% through iterative feedback.Related work on general-purpose language models has also explored
% iterative self-feedback, verbal reflection, and tool-interactive
% self-correction \cite{madaan2023selfrefine,shinn2023reflexion,
% gou2024critic}.ProTeGi~\cite{pryzant2023automatic}
% uses natural-language gradients and beam search for
% prompt refinement, while OPRO~\cite{yang2024opro}
% generates candidate instructions based on previous
% evaluation results. TextGrad~\cite{textgrad}
% extends language-based optimization to compound AI
% systems by propagating feedback across textual variables
% through textual backpropagation. 
Language-based refinement has been explored through
iterative self-feedback, verbal reflection, and
tool-assisted correction
\cite{madaan2023selfrefine,shinn2023reflexion,gou2024critic}.
For prompt optimization, ProTeGi~\cite{pryzant2023automatic}
employs natural-language gradients and beam search,
while OPRO~\cite{yang2024opro} uses language models
to optimize instructions based on previous evaluations.
TextGrad~\cite{textgrad} further extends textual
optimization to compound AI systems through
language-based backpropagation across textual variables.
These methods provide
general mechanisms for optimizing instructions without
differentiating through underlying models. Applying
such optimization to segmentation additionally requires
grounding visual errors in relevant instruction variables
and maintaining task-specific constraints during updates.

\section{Methodology}
\label{sec:method}

\begin{figure*}[t]
    \centering
    \includegraphics[width=0.90\linewidth]
    {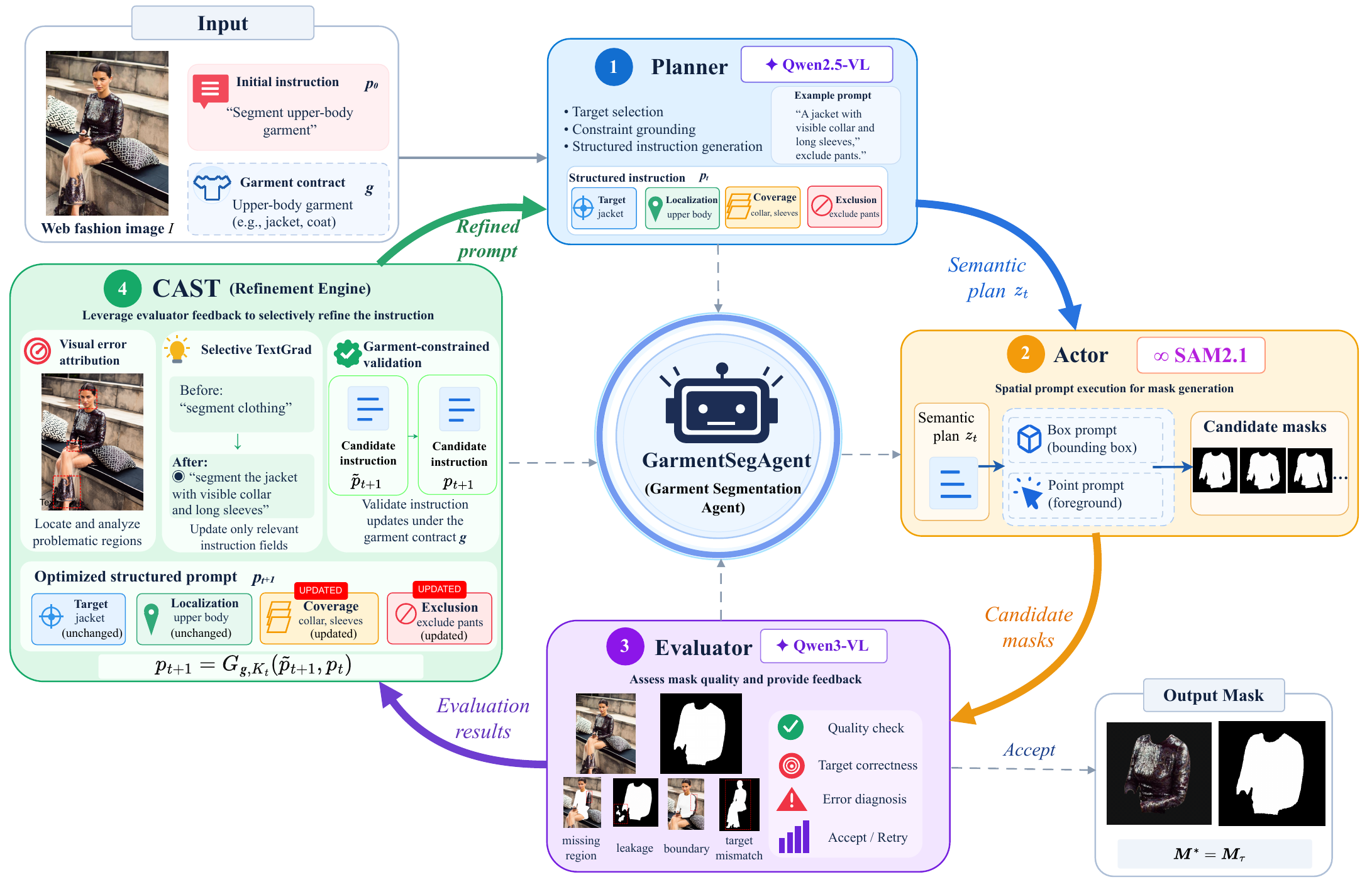}
   \caption{Overview of GarmentSegAgent.
    The Planner generates a garment-aware semantic plan,
    the Actor produces candidate masks, and the Evaluator
    selects predictions and provides visual feedback.
    CAST performs selective, garment-constrained textual
    refinement to guide subsequent segmentation iterations.}
    \captionsetup{skip=3pt}
    \label{fig:framework}
\end{figure*}

\subsection{Problem Formulation and Framework Overview}
\label{sec:problem_formulation}

% We investigate autonomous garment segmentation for heterogeneous Web fashion images. Given an input image
% $I \in \mathbb{R}^{H \times W \times 3}$ and an initial textual instruction $p_0$, our objective is to identify and extract a
% target garment instance as a binary segmentation mask $M^* \in \{0,1\}^{H \times W}$, without requiring human-provided
% spatial prompts. We formulate this task as an iterative decision-making process, where segmentation actions are
% autonomously planned, evaluated, and refined through visual feedback during inference.

We consider autonomous garment segmentation in heterogeneous
Web fashion images. Given an input image
$I \in \mathbb{R}^{H \times W \times 3}$ and an initial
textual instruction $p_0$, the objective is to extract
a target garment instance as a binary segmentation mask
$M^* \in \{0,1\}^{H \times W}$ without human-provided
spatial prompts. We formulate this task as an iterative
process in which segmentation predictions are evaluated
and refined through visual feedback at inference time.

% To maintain task consistency across iterations, we introduce a garment extraction contract $g$ that specifies target
% selection and segmentation constraints. Given an instruction $p$, the Planner $\mathcal{P}$ generates spatial prompts,
% which are executed by the segmentation model $\mathcal{S}$ to produce candidate masks:
% \begin{equation}
%     \mathcal{M}(p) = \mathcal{S}(I,\mathcal{P}(I,p,g)).
% \end{equation}
% An evaluator selects a candidate mask $M(p)$ from $\mathcal{M}(p)$. We formulate textual self-refinement as a
% constrained test-time optimization problem:
% \begin{equation}
%     \max_{p \in \mathcal{C}(g)}
%     \widehat{Q}(I,M(p),g),
% \end{equation}
% where $\widehat{Q}$ denotes the evaluator-estimated segmentation quality, and $\mathcal{C}(g)$ represents the
% set of textual instructions satisfying the garment contract.

To maintain task consistency, we introduce a garment
extraction contract $g$ specifying target-selection
rules and segmentation constraints. Given an instruction
$p$, the Planner $\mathcal{P}$ generates a semantic
segmentation plan, which the Actor $\mathcal{S}$
converts into spatial prompts to produce candidate masks:
\begin{equation}
    \mathcal{M}(p)
    = \mathcal{S}(I,\mathcal{P}(I,p,g)).
\end{equation}
The visual Evaluator selects a mask $M(p)$ from
$\mathcal{M}(p)$ and provides feedback on segmentation
errors. We formulate textual self-refinement as a
constrained test-time optimization problem:
\begin{equation}
    \max_{p \in \mathcal{C}(g)}
    \widehat{Q}(I,M(p),g),
\end{equation}
where $\widehat{Q}$ denotes evaluator-estimated
segmentation quality and $\mathcal{C}(g)$ is the set
of instructions satisfying the garment contract.
This objective guides feedback-driven refinement
under a finite inference budget.

% To address this problem, we propose \textbf{GarmentSegAgent}, an autonomous framework integrating garment-aware planning,
% spatial segmentation, visual evaluation, and textual self-refinement, as illustrated in Fig.~\ref{fig:framework}.
% At its core, we introduce \textbf{CAST} (\textbf{C}onstrained \textbf{A}ttribution-guided
% \textbf{S}elective \textbf{T}extGrad), a task-aware textual optimization mechanism that attributes visual segmentation
% errors to relevant instruction components and selectively updates them through textual gradients. CAST further enforces
% garment-specific constraints to preserve task consistency during refinement. The resulting instructions guide subsequent
% segmentation iterations, enabling autonomous error correction while keeping all foundation models frozen.

To realize this formulation, we develop
\textbf{GarmentSegAgent}, which integrates semantic
planning, spatial execution, visual evaluation,
and CAST-based textual optimization into a closed-loop
framework, as shown in Fig.~\ref{fig:framework}.
CAST attributes visual errors to relevant instruction
components and selectively refines them through
TextGrad under garment-specific constraints.
The validated instructions guide subsequent segmentation
iterations, while all foundation models remain frozen.

% We propose SegAgent, an TextGrad-Guided Evaluator Agent for Upper-Body
% Garment Segmentation in complex RGB images. Unlike conventional segmentation pipelines that rely on fixed prompts or directly optimize segmentation masks, SegAgent formulates garment extraction as a self-refining decision process. The proposed framework enables an intelligent agent to iteratively perceive the visual environment, execute segmentation actions, evaluate the generated results, and improve its future behavior through language-level optimization.

% As illustrated in Fig.~\ref{fig:framework}, SegAgent consists of four collaborative modules:
% (1) Planner, (2) Actor, (3) Evaluator, and (4) TextGrad-based Feedback Optimization.The Planner performs semantic understanding and generates executable segmentation strategies.The Actor converts these strategies into spatial prompts and produces candidate garment masks. The Evaluator examines the segmentation results through visual evidence and provides structured quality feedback. Finally, the Feedback module employs TextGrad to transform visual feedback into textual gradients, enabling automatic refinement of segmentation instructions.

% Different from traditional feedback mechanisms that only provide confidence scores or heuristic corrections, SegAgent introduces a language-level optimization mechanism. By treating segmentation instructions as optimizable variables, TextGrad allows the agent to continuously improve its segmentation policy without updating the parameters of the underlying vision-language and segmentation models.

\subsection{Garment-Aware Planning and Spatial Execution}
\label{sec:planning_execution}

% To support targeted textual refinement, we organize the segmentation instruction into four components:
% $p_t=\{p_t^{tar},p_t^{loc},p_t^{cov},p_t^{exc}\}$,
% corresponding to target selection, spatial localization, garment coverage, and region exclusion, respectively.
% These components serve as independently optimizable textual variables, while the garment extraction contract
% $g$ specifies the task-level constraints that must be preserved throughout refinement.

We represent the segmentation instruction at iteration
$t$ as four structured textual variables:
$p_t=(p_t^{tar},p_t^{loc},p_t^{cov},p_t^{exc})$,
corresponding to target selection, spatial localization,
garment coverage, and region exclusion, respectively.
These variables allow individual instruction components
to be selectively refined while maintaining consistency
with the garment extraction contract $g$.

% Given the input image $I$, instruction $p_t$, and contract $g$, the Planner employs Qwen2.5-VL to identify the target garment and generate a semantic segmentation plan:
% \begin{equation}
%     z_t = \mathcal{P}(I,p_t,g)
%     = (c_t,b_t,\mathcal{R}_t^+,\mathcal{R}_t^-),
% \end{equation}
% where $c_t$ denotes the target garment description, $b_t$ is its estimated bounding box, and
% $\mathcal{R}_t^+$ and $\mathcal{R}_t^-$ represent candidate foreground and exclusion regions, respectively.
% This plan translates textual instructions into spatial guidance for subsequent segmentation.

Given the image $I$, current instruction $p_t$, and
contract $g$, the Planner employs Qwen2.5-VL~\cite{qwen25vl2025} to
generate a garment-aware semantic plan:
\begin{equation}
    z_t = \mathcal{P}(I,p_t,g)
    = (c_t,b_t,\mathcal{R}_t^+,\mathcal{R}_t^-),
\end{equation}
where $c_t$ describes the target garment, $b_t$ denotes
its estimated bounding box, and $\mathcal{R}_t^+$ and
$\mathcal{R}_t^-$ specify candidate foreground and
exclusion regions, respectively.

% The Actor converts the semantic plan into executable spatial prompts and generates candidate masks using
% SAM2.1. Specifically, the bounding box and validated positive and negative points are used to produce
% segmentation candidates:
% \begin{equation}
%     \mathcal{M}_t =
%     \left\{
%     \mathcal{S}(I,b_t,P_{t,k}^+,P_{t,k}^-)
%     \right\}_{k=1}^{K_t},
% \end{equation}
% where $P_{t,k}^+$ and $P_{t,k}^-$ denote the validated positive and negative point sets for the $k$-th prompt
% configuration, and $K_t$ is the number of candidates. The resulting masks are subsequently assessed by the
% visual evaluator to identify segmentation errors and guide textual self-refinement.

The Actor translates this semantic plan into spatial
prompts for SAM 2.1. Using the estimated bounding box
and positive and negative point configurations derived
from the corresponding regions, it generates a set
of candidate masks:
\begin{equation}
    \mathcal{M}_t =
    \left\{
    \mathcal{S}(I,b_t,P_{t,k}^+,P_{t,k}^-)
    \right\}_{k=1}^{B_t},
\end{equation}
where $P_{t,k}^+$ and $P_{t,k}^-$ denote the positive
and negative point sets for the $k$-th prompt
configuration, and $B_t$ is the number of candidate
configurations. The resulting masks are passed to
the Evaluator for candidate selection and visual
error diagnosis.

\subsection{Visual Error Attribution}
\label{sec:visual_error_attribution}

% While garment-aware planning enables autonomous segmentation, the generated masks may still suffer from incorrect target selection, incomplete garment coverage, or inaccurate boundaries. Conventional visual feedback typically provides an overall quality assessment, offering limited guidance on which segmentation decisions should be corrected. To address this limitation, we introduce Visual Error Attribution (VEA),
% which transforms visual segmentation failures into structured, actionable feedback for subsequent textual optimization.

Despite garment-aware planning, the generated masks
may contain target mismatches, incomplete garment
regions, or inaccurate boundaries. A coarse quality
assessment alone provides limited guidance for
targeted correction. We therefore introduce
Visual Error Attribution (VEA) to characterize
segmentation errors as structured visual feedback
for subsequent textual optimization.

% Given the input image $I$, candidate masks $\mathcal{M}_t$, current instruction $p_t$, and garment contract $g$, the visual evaluator examines target consistency, garment coverage, region exclusion, and boundary accuracy. The evaluation process is formulated as
% \begin{equation}
%     (M_t,e_t) = \mathcal{E}(I,\mathcal{M}_t,p_t,g),
% \end{equation}
% where $\mathcal{E}$ denotes the visual evaluator, $M_t$ is the selected candidate mask, and $e_t$ represents its structured error feedback. Specifically, VEA considers
% four categories of segmentation errors:
% \textit{target mismatch}, \textit{missing regions},
% \textit{region leakage}, and \textit{boundary inaccuracies}.
% Rather than relying solely on an overall quality judgment, VEA characterizes each identified error as
% \begin{equation}
%     e_t = \left\{
%     (k_{t,i},r_{t,i},s_{t,i},c_{t,i},a_{t,i})
%     \right\}_{i=1}^{N_t},
% \end{equation}
% where $k_{t,i}$ denotes the error category, $r_{t,i}$ specifies the affected visual region, $s_{t,i}$ indicates error severity, $c_{t,i}$ represents evaluator confidence,
% and $a_{t,i}$ describes the suggested correction.

Given the image $I$, candidate masks $\mathcal{M}_t$,
current instruction $p_t$, and garment contract $g$,
the visual Evaluator selects a candidate mask and
generates error-specific feedback:
\begin{equation}
    (M_t,e_t) = \mathcal{E}(I,\mathcal{M}_t,p_t,g),
\end{equation}
where $M_t \in \mathcal{M}_t$ is the selected mask and
$e_t$ denotes its structured visual feedback.
VEA considers four error categories:
\textit{target mismatch}, \textit{missing regions},
\textit{region leakage}, and \textit{boundary errors}.
Each identified error is represented by
\begin{equation}
    e_t = \left\{
    (k_{t,i},r_{t,i},s_{t,i},c_{t,i},a_{t,i})
    \right\}_{i=1}^{N_t},
\end{equation}
where $N_t$ is the number of identified errors,
$k_{t,i}$ denotes the error category, $r_{t,i}$
specifies the affected visual region, $s_{t,i}$
and $c_{t,i}$ represent error severity and evaluator
confidence, respectively, and $a_{t,i}$ describes
a suggested correction.

% Importantly, VEA identifies segmentation failures directly from visual evidence without accessing ground-truth masks. For example, an omitted sleeve can be attributed to
% incomplete garment coverage, whereas incorrectly retained skin regions indicate insufficient exclusion constraints. Such structured feedback provides explicit correction
% targets rather than unrestricted natural-language critiques. The resulting error signals are subsequently used by CAST to identify relevant instruction components and guide
% selective textual refinement.

VEA infers these errors from the input image and
predicted masks without accessing ground-truth
annotations. For instance, a missing sleeve indicates
insufficient garment coverage, whereas retained skin
regions suggest inadequate exclusion. These structured
diagnostics provide error-specific signals that CAST maps to relevant instruction variables
for selective textual refinement.

\subsection{CAST: Constrained Attribution-Guided Selective TextGrad}
\label{sec:cast}

% \noindent\textbf{Constrained Textual Optimization.}
% TextGrad~\cite{textgrad} enables the optimization of textual variables by propagating language-model-generated feedback. However, directly optimizing an entire segmentation
% instruction may modify previously valid components when correcting localized visual errors. To address this issue, we propose \textbf{CAST}, an attribution-guided textual
% optimization mechanism that selectively updates error-related instruction components while preserving garment-specific constraints.

\noindent\textbf{Constrained Textual Optimization.}
TextGrad~\cite{textgrad} enables feedback-driven
optimization of textual variables, but unrestricted
instruction updates may disrupt previously valid
components when correcting localized segmentation errors.
To address this issue, we introduce \textbf{CAST},
which couples visual error attribution with selective
textual optimization under garment-specific constraints.

% Given the structured instruction 
% $p_t=\{p_t^{(j)}\}_{j\in\mathcal{J}}$,
% where $\mathcal{J}=\{tar,loc,cov,exc\}$, CAST formulates each refinement step as a constrained optimization problem:
% \begin{equation}
% \begin{aligned}
% \max_{p}\quad & \widehat{Q}(I,M(p),g)\\
% \text{s.t.}\quad
% & p\in\mathcal{C}(g),\\
% & p^{(j)}=p_t^{(j)},\quad j\notin\mathcal{K}_t.
% \end{aligned}
% \label{eq:cast_objective}
% \end{equation}
% where $\mathcal{K}_t\subseteq\mathcal{J}$ denotes the instruction components selected for refinement. Unlike unrestricted prompt optimization, this formulation restricts updates to error-relevant variables while maintaining the remaining instruction components.

Let $\mathcal{J}=\{tar,loc,cov,exc\}$ index the four
instruction fields defined in Section~\ref{sec:planning_execution},
and let $\mathcal{K}_t\subseteq\mathcal{J}$ denote
the error-relevant fields selected at iteration $t$.
We formulate the desired refinement as
\begin{equation}
\begin{aligned}
\max_{p}\quad & \widehat{Q}(I,M(p),g)\\
\text{s.t.}\quad
& p\in\mathcal{C}(g),\\
& p^{(j)}=p_t^{(j)},
\quad \forall j\in\mathcal{J}\setminus\mathcal{K}_t.
\end{aligned}
\label{eq:cast_objective}
\end{equation}
This formulation restricts updates to selected
instruction fields while requiring the resulting
instruction to satisfy the garment contract.
CAST operationalizes this objective through
error-to-instruction routing, selective TextGrad
updates, and garment-constrained 
validation.

% \noindent\textbf{Error-to-Instruction Routing.}
% To determine the optimizable variables, CAST maps the structured visual feedback $e_t$ into a set of relevant instruction components:
% \begin{equation}
%     \mathcal{K}_t =
%     \bigcup_{i=1}^{N_t}\mathcal{R}(k_{t,i}),
% \label{eq:error_routing}
% \end{equation}
% where $\mathcal{R}$ denotes the error-to-instruction routing function. Specifically, target mismatches activate target selection and localization instructions;
% missing garment regions activate coverage and localization instructions; region leakage activates exclusion and localization instructions; and boundary inaccuracies
% activate coverage and localization instructions. This mapping establishes an explicit connection between visual segmentation failures and textual optimization
% variables.

\noindent\textbf{Error-to-Instruction Routing.}
CAST identifies the instruction fields requiring refinement
based on the error categories detected by VEA.
Given the structured feedback $e_t$, the active
instruction set is determined by
\begin{equation}
    \mathcal{K}_t =
    \bigcup_{i=1}^{N_t}\mathcal{R}(k_{t,i}),
\label{eq:error_routing}
\end{equation}
where $\mathcal{R}$ maps each error category to a subset
of instruction fields in $\mathcal{J}$. Specifically,
target mismatches activate target and localization;
missing garment regions activate coverage and localization;
region leakage activates exclusion and localization;
and boundary inaccuracies activate coverage and localization.
When multiple errors are detected, their associated fields
are combined, while unrelated instruction components
remain unchanged during refinement.

\noindent\textbf{Selective Textual Gradient Descent.}
% Following TextGrad~\cite{yuksekgonul2025optimizing}, we formulate instruction refinement through textual loss evaluation, language-based backpropagation, and textual gradient descent. Unlike conventional TextGrad optimization with a predefined set of variables,
% CAST dynamically selects the optimizable instruction components according to visual error attribution.
% 
% For each selected component $j\in\mathcal{K}_t$, we collect the associated visual errors as
% $e_t^{(j)}=\{e_{t,i}\mid
% j\in\mathcal{R}(k_{t,i})\}$.
% We then construct a component-specific textual loss:
% \begin{equation}
%     \ell_t^{(j)}
%     = \mathcal{L}_{\mathrm{text}}
%     (p_t^{(j)};e_t^{(j)},p_t,g),
% \label{eq:cast_loss}
% \end{equation}
% where $\mathcal{L}_{\mathrm{text}}$ denotes a language-model-based evaluation operation that assesses the selected instruction against the attributed segmentation errors.
Given the selected instruction fields $\mathcal{K}_t$,
CAST applies TextGrad
to refine only the variables associated with
detected visual errors. For each $j\in\mathcal{K}_t$,
the corresponding error feedback is collected as
$e_t^{(j)}=\{e_{t,i}\mid j\in\mathcal{R}(k_{t,i})\}$.
We construct a component-specific textual loss:
\begin{equation}
    \ell_t^{(j)}
    = \mathcal{L}_{\mathrm{text}}
    (p_t^{(j)};e_t^{(j)},p_t,g),
\label{eq:cast_loss}
\end{equation}
where $\mathcal{L}_{\mathrm{text}}$ produces an
error-conditioned assessment of the selected field
in the context of the full instruction and garment
contract.
% TextGrad subsequently propagates the textual loss through its language-based backward operation:
% \begin{equation}
%     \gamma_t^{(j)}
%     = \mathcal{B}_{\mathrm{TG}}
%     (\ell_t^{(j)};p_t^{(j)}),
% \label{eq:cast_gradient}
% \end{equation}
% where $\mathcal{B}_{\mathrm{TG}}$ denotes the textual backward operator and $\gamma_t^{(j)}$ represents the resulting natural-language gradient.
% Rather than computing numerical derivatives, the backward operation generates corrective
% feedback for improving the selected instruction.
The textual loss is propagated through the
language-based backward operator:
\begin{equation}
    \gamma_t^{(j)}
    = \mathcal{B}_{\mathrm{TG}}
    (\ell_t^{(j)};p_t^{(j)}),
\label{eq:cast_gradient}
\end{equation}
where $\mathcal{B}_{\mathrm{TG}}$ denotes the textual backward operator and $\gamma_t^{(j)}$ denotes the natural-language
gradient identifying potential corrections to the
selected instruction field. Unlike numerical gradients,
this feedback provides semantic guidance for
instruction revision.

% The resulting gradients are applied through selective textual gradient descent:
% \begin{equation}
% \tilde{p}_{t+1}^{(j)} =
% \begin{cases}
% \mathcal{U}_{\mathrm{TGD}}
% (p_t^{(j)},\gamma_t^{(j)}),
% & j\in\mathcal{K}_t,\\
% p_t^{(j)}, & j\notin\mathcal{K}_t,
% \end{cases}
% \label{eq:cast_update}
% \end{equation}
% where $\mathcal{U}_{\mathrm{TGD}}$ represents the TextGrad optimizer step. By dynamically
% restricting the updated variables to $\mathcal{K}_t$, CAST performs error-conditioned
% textual optimization while preserving unselected instruction components.

The resulting gradients drive selective textual updates:
\begin{equation}
\tilde{p}_{t+1}^{(j)} =
\begin{cases}
\mathcal{U}_{\mathrm{TGD}}
(p_t^{(j)},\gamma_t^{(j)}),
& j\in\mathcal{K}_t,\\
p_t^{(j)}, & j\notin\mathcal{K}_t,
\end{cases}
\label{eq:cast_update}
\end{equation}
where $\mathcal{U}_{\mathrm{TGD}}$ denotes the
TextGrad optimizer step. This operation updates
only error-relevant fields while leaving the
remaining components unchanged. The resulting
candidate instruction $\tilde{p}_{t+1}$ is subsequently
validated against the garment contract

\noindent\textbf{Garment-Constrained Update.}
Although selective TextGrad restricts instruction updates
to error-relevant fields, the resulting candidates may
still violate garment-specific requirements. We therefore
introduce a contract-aware validation operator:
\begin{equation}
    p_{t+1} =
    \mathcal{G}_{g,\mathcal{K}_t}
    (\tilde{p}_{t+1},p_t),
\label{eq:cast_constraint}
\end{equation}
where $\mathcal{G}_{g,\mathcal{K}_t}$ checks the candidate
instruction against the garment contract and update
scope. Compliant updates are accepted, while violations
are repaired when feasible or rolled back
to $p_t$. This process aims to maintain target identity,
coverage, and exclusion requirements without
altering unselected instruction 
fields.

Together, visual error attribution, selective TextGrad,
and garment-constrained validation enable CAST to
refine instructions based on localized segmentation
errors. The validated instruction $p_{t+1}$ is returned
to the Planner for the next segmentation iteration,
while all foundation models remain frozen.
Theoretical analysis of update locality and conditional
constraint preservation is provided in
Appendix~\ref{app:theoretical_analysis}.

\subsection{Iterative Decision and Termination}
\label{sec:iterative_decision}

% GarmentSegAgent integrates the proposed CAST optimization into an iterative segmentation process. At each iteration, the Evaluator assesses the selected mask $M_t$ and its
% associated visual error feedback $e_t$ to determine whether further refinement is necessary. The decision is defined as
% \begin{equation}
%     d_t = \mathcal{D}(I,M_t,e_t,g)
%     \in \{\mathrm{accept},\mathrm{retry}\},
% \label{eq:decision}
% \end{equation}
% where $\mathcal{D}$ denotes the evaluator-based decision function. If the predicted mask satisfies the garment extraction requirements, the agent accepts the result.
% Otherwise, CAST performs error-attributed, selective textual optimization under garment-specific constraints.

At each iteration, the Evaluator determines whether
the selected mask $M_t$ requires further refinement
based on its visual error feedback $e_t$ and the
garment contract $g$:
\begin{equation}
    d_t = \mathcal{D}(I,M_t,e_t,g)
    \in \{\mathrm{accept},\mathrm{retry}\},
\label{eq:decision}
\end{equation}
where $\mathcal{D}$ denotes the evaluator-based
decision function. A mask is accepted when the
Evaluator judges it consistent with the garment
extraction requirements. Otherwise, CAST generates
a validated instruction $p_{t+1}$, which guides
the Planner and Actor in producing new segmentation
candidates. The complete inference procedure is
summarized in Algorithm~\ref{alg:garmentsegagent}.

% The refined instruction $p_{t+1}$ is subsequently returned to the Planner to generate updated spatial prompts, which guide the Actor in producing new segmentation candidates. This process establishes a closed-loop interaction between visual assessment and
% textual optimization, allowing GarmentSegAgent to progressively refine its segmentation decisions without human intervention or model parameter updates. The overall inference procedure of GarmentSegAgent
% is summarized in Algorithm~\ref{alg:garmentsegagent}.

% To ensure bounded inference, we impose a maximum of $T$ segmentation iterations. The process terminates when the Evaluator accepts a mask or the iteration budget is exhausted. Formally, the termination iteration $\tau$ and final output are defined as
% \begin{equation}
% \begin{aligned}
%     \tau &= \min\left(
%     \{t \mid d_t=\mathrm{accept}\}
%     \cup \{T-1\}\right),\\
%     M^* &= M_{\tau},
% \end{aligned}
% \label{eq:termination}
% \end{equation}
% where $t\in\{0,\ldots,T-1\}$. Throughout inference, all foundation models remain
% frozen, and the textual optimization process is performed independently for each input image.

To bound inference cost, we allow at most $T$
segmentation iterations. The process terminates
upon acceptance or when the iteration budget is
exhausted. The stopping iteration $\tau$ and final
segmentation mask are defined as
\begin{equation}
\begin{aligned}
    \tau &= \min\left(
    \{t \mid d_t=\mathrm{accept}\}
    \cup \{T-1\}\right),\\
    M^* &= M_{\tau},
\end{aligned}
\label{eq:termination}
\end{equation}
where $t\in\{0,\ldots,T-1\}$. If no mask is accepted,
the final candidate is returned. Each image is
processed independently, with all foundation
models remaining frozen.

% \begin{algorithm}[t]
% \caption{GarmentSegAgent with CAST Optimization}
% \label{alg:garmentsegagent}
% \begin{algorithmic}[1]
% \Require Image $I$, initial instruction $p_0$,
%          garment contract $g$, maximum iterations $T$
% \Ensure Garment segmentation mask $M^*$

% \For{$t=0,\ldots,T-1$}
%     \State $z_t \gets \mathcal{P}(I,p_t,g)$
%     \State $\mathcal{M}_t \gets \mathcal{S}(I,z_t)$
%     \State $(M_t,e_t) \gets
%            \mathcal{E}(I,\mathcal{M}_t,p_t,g)$
%     \State $d_t \gets \mathcal{D}(I,M_t,e_t,g)$

%     \If{$d_t=\mathrm{accept}$ \textbf{or} $t=T-1$}
%         \State \Return $M_t$
%     \EndIf

%     \State $\mathcal{K}_t \gets
%            \bigcup_{i=1}^{N_t}\mathcal{R}(k_{t,i})$

%     \State $\tilde{p}_{t+1} \gets
%            \mathrm{SelectiveTGD}(p_t,e_t,\mathcal{K}_t,g)$

%     \State $p_{t+1} \gets
%            \mathcal{G}_{g,\mathcal{K}_t}
%            (\tilde{p}_{t+1},p_t)$
% \EndFor
% \end{algorithmic}
% \end{algorithm}

\begin{algorithm}[t]
\caption{GarmentSegAgent with CAST Optimization}
\label{alg:garmentsegagent}
\begin{algorithmic}[1]
\Require Image $I$, initial instruction $p_0$,
         garment contract $g$, iteration budget $T$
\Ensure Segmentation mask $M^*$

\For{$t=0,\ldots,T-1$}
    \State $z_t \gets \mathcal{P}(I,p_t,g)$
    \State $\mathcal{M}_t \gets \mathcal{S}(I,z_t)$
    \State $(M_t,e_t) \gets
           \mathcal{E}(I,\mathcal{M}_t,p_t,g)$
    \State $d_t \gets \mathcal{D}(I,M_t,e_t,g)$

    \If{$d_t=\mathrm{accept}$ \textbf{or} $t=T-1$}
        \State \Return $M_t$
    \EndIf

    \State $\mathcal{K}_t \gets
           \bigcup_{i=1}^{N_t}\mathcal{R}(k_{t,i})$
    \State $\tilde{p}_{t+1} \gets p_t$

    \For{$j\in\mathcal{K}_t$}
        \State $e_t^{(j)} \gets
               \{e_{t,i}\mid j\in\mathcal{R}(k_{t,i})\}$
        \State $\ell_t^{(j)} \gets
               \mathcal{L}_{\mathrm{text}}
               (p_t^{(j)};e_t^{(j)},p_t,g)$
        \State $\gamma_t^{(j)} \gets
               \mathcal{B}_{\mathrm{TG}}
               (\ell_t^{(j)};p_t^{(j)})$
        \State $\tilde{p}_{t+1}^{(j)} \gets
               \mathcal{U}_{\mathrm{TGD}}
               (p_t^{(j)},\gamma_t^{(j)})$
    \EndFor

    \State $p_{t+1} \gets
           \mathcal{G}_{g,\mathcal{K}_t}
           (\tilde{p}_{t+1},p_t)$
\EndFor
\end{algorithmic}
\end{algorithm}

% \subsection{Iterative Decision and Termination}

% At each iteration, SegAgent receives the refined instruction, regenerates
% spatial prompts, and produces updated segmentation candidates.

% The decision process is defined as:

% \begin{equation}
% d_t\in\{\mathrm{accept},\mathrm{retry}\}.
% \end{equation}

% If the Evaluator confirms that the current segmentation satisfies the
% garment extraction requirements, the process terminates.
% Otherwise, the feedback module generates an improved instruction and initiates
% another segmentation iteration.

% The final output mask is obtained from the accepted segmentation result:

% \begin{equation}
% M^{*}=M_{\tau},
% \end{equation}

% where $\tau$ denotes the final executed iteration.

% Throughout inference, all foundation models remain frozen.
% Only instructions, spatial prompts, and intermediate execution states are
% updated during the agentic refinement process.

% Add Experiments, Discussion, and Conclusion here.
% No experimental results are fabricated in this Method-only draft.

% Enable these lines after adding your actual bibliography file.
% \bibliographystyle{ACM-Reference-Format}
% \bibliography{references}

% Add optional supplementary material after \appendix.
\section{Experiments}
\label{sec:experiments}

\subsection{Experimental Setup} \label{sec:experimental_setup}

% \subsubsection{Datasets and Evaluation Protocol}

% We evaluate SegAgent on challenging upper-body garment segmentation
% scenarios collected from real-world fashion images.

% Different from conventional garment segmentation benchmarks that mainly
% contain clean product images, our evaluation focuses on complex RGB images
% with diverse backgrounds, human-object interactions, and visual
% interference.

% We construct four evaluation cohorts:

% \begin{itemize}

% \item
% \textbf{DeepFashion2 upper-garment segmentation:}
% A large-scale fashion benchmark containing diverse clothing categories
% and real-world human wearing scenarios.

% \item
% \textbf{ModaNet upper-garment segmentation:}
% A street-fashion dataset containing complex backgrounds and large
% appearance variations.

% \item
% \textbf{Fashionpedia upper-garment verification:}
% A fine-grained fashion benchmark used for analyzing
% TextGrad-based refinement behaviors.

% \item
% \textbf{COCO person segmentation:}
% A general object segmentation benchmark used to evaluate the
% generalization ability of the proposed agentic framework.

% \end{itemize}

% For each image, SegAgent predicts a binary foreground mask.
% All predictions are evaluated offline against the corresponding
% ground-truth annotations.

% During inference, no ground-truth mask, category label,
% or evaluation metric is accessible to the Planner,
% Actor, Evaluator, or refinement module.

\subsubsection{Datasets and Evaluation Protocol}

% We evaluate GarmentSegAgent on three established fashion
% segmentation benchmarks: DeepFashion2~\cite{ge2019deepfashion2},
% ModaNet~\cite{zheng2018modanet}, and
% Fashionpedia~\cite{jia2020fashionpedia}. DeepFashion2 contains
% both consumer and commercial fashion images with substantial
% variations in garment appearance, viewpoints, and occlusions.
% ModaNet comprises street-fashion images featuring cluttered
% backgrounds and diverse clothing configurations, whereas
% Fashionpedia provides fine-grained garment annotations
% covering a wide range of clothing categories and attributes.
% Together, these benchmarks enable evaluation across
% heterogeneous fashion imagery and garment structures.

We evaluate GarmentSegAgent on three fashion segmentation
benchmarks: DeepFashion2~\cite{ge2019deepfashion2},
ModaNet~\cite{zheng2018modanet}, and
Fashionpedia~\cite{jia2020fashionpedia}.
DeepFashion2 includes consumer and commercial images
with diverse garment appearances, viewpoints, and
occlusions. ModaNet features street-fashion imagery
with cluttered backgrounds and varied clothing
configurations, while Fashionpedia provides fine-grained
annotations across garment categories and attributes.
Together, these datasets cover diverse visual conditions
and garment characteristics.

% We adopt a unified upper-garment segmentation setting,
% using 500 evaluation images from each benchmark. For each
% image, the task is to extract a target upper-garment instance
% as a binary segmentation mask. When multiple garments are
% present, a consistent target-selection and annotation-matching
% protocol is applied across the evaluated methods, with
% dataset-specific details provided in Appendix~\ref{app:implementation}.
% All images are processed independently without human-provided
% spatial prompts or access to ground-truth masks during inference.
% The Planner, Actor, Evaluator, and CAST rely exclusively on
% the input image and intermediate visual feedback, while
% ground-truth annotations and evaluation metrics are used
% only for offline assessment. 
% Additional details on target selection and annotation
% matching are provided in Appendix~\ref{app:implementation}.

We adopt a unified upper-garment segmentation setting,
evaluating 500 images from each benchmark. The objective
is to extract a target upper-garment instance as a binary
mask. For images containing multiple garments, a
consistent target-selection and annotation-matching
protocol determines the evaluation target, as detailed
in Appendix~\ref{app:implementation}.
GarmentSegAgent processes each image independently
without human-provided spatial prompts or access to
ground-truth annotations. Its refinement relies solely
on the input image and intermediate visual feedback,
with ground-truth masks and evaluation metrics reserved
for offline assessment.

% \subsubsection{Baselines}

% We compare GarmentSegAgent against representative methods
% spanning three segmentation paradigms: spatial prompting,
% vision-language segmentation, and reasoning-based segmentation.
% The first group includes SAM 2~\cite{ravi2025sam2} with box
% prompts and Grounding DINO+SAM 2~\cite{liu2024groundingdino},
% which combines object detection with promptable segmentation.
% The second group comprises READ~\cite{qian2025reasoning},
% SEEM~\cite{zou2023segment}, and LISA~\cite{lai2024lisa},
% which incorporate language-guided visual understanding
% into segmentation. We further consider SegZero+SAM 2
% ~\cite{liu2025seg} and SegAgent~\cite{zhu2025segagent}
% as reasoning-based and agentic segmentation baselines,
% respectively. All methods are evaluated on the same image
% subsets and target garment definitions. Given their
% different prompting requirements, we distinguish
% externally supplied spatial prompts from automatically
% generated ones when interpreting the results.

\subsubsection{Baselines}

We compare GarmentSegAgent with representative baselines
across spatial prompting, vision-language segmentation,
and reasoning-based or agentic segmentation.
Spatial prompting methods include SAM 2~\cite{ravi2025sam2}
with box prompts and Grounding DINO+SAM 2
~\cite{liu2024groundingdino}, which generates spatial
prompts through object detection. Vision-language
segmentation methods include READ~\cite{qian2025reasoning},
SEEM~\cite{zou2023segment}, and LISA~\cite{lai2024lisa}.
We further include SegZero+SAM 2~\cite{liu2025seg}
and SegAgent~\cite{zhu2025segagent} as reasoning-based
and agentic segmentation baselines, respectively.
All methods are evaluated on identical image subsets
and target garment definitions. We distinguish
externally supplied spatial prompts from automatically
generated prompts to account for differences in
prompting requirements.

% \subsubsection{Implementation Details}

% GarmentSegAgent employs Qwen2.5-VL as the Planner,
% SAM 2.1 as the segmentation Actor, and Qwen3-VL as
% the visual Evaluator. The Planner generates garment-aware
% spatial prompts, including bounding boxes and foreground
% and background points, which are executed by SAM 2.1
% to produce candidate masks. The Evaluator assesses these
% predictions and provides structured visual error feedback.
% For textual refinement, CAST organizes segmentation
% instructions into four components: target, localization,
% coverage, and exclusion. Based on the attributed errors,
% TextGrad updates only the selected instruction components,
% while garment-contract validation checks the proposed
% updates for task consistency

% All foundation models remain frozen throughout inference.
% Each image is processed independently, with textual
% instructions and corresponding spatial prompts refined
% at test time without parameter updates or cross-image
% state sharing. The maximum number of segmentation
% iterations is set to five, with early termination
% when the Evaluator accepts the predicted mask.

\subsubsection{Implementation Details}

GarmentSegAgent employs Qwen2.5-VL as the Planner,
SAM 2.1 as the segmentation Actor, and Qwen3-VL~\cite{qwen3vl2025}
as the visual Evaluator. The Planner generates a
garment-aware semantic plan, which the Actor translates
into bounding-box and positive/negative point prompts
for candidate mask generation. The Evaluator selects
a mask and provides structured visual feedback to CAST,
which selectively refines the target, localization,
coverage, and exclusion instruction fields under
garment-specific constraints. All foundation models
remain frozen, and each image is processed independently
without cross-image state sharing. We allow at most
five segmentation iterations, including the initial
prediction, with early termination upon evaluator
acceptance.

\subsubsection{Evaluation Metrics}

We assess segmentation quality using mean Intersection over
Union (mIoU), Dice coefficient~\cite{milletari2016vnet}, Boundary IoU (BIoU)~\cite{cheng2021boundary},
precision, and recall. mIoU and Dice quantify region-level
agreement between predicted and ground-truth garment masks,
while BIoU measures boundary alignment, particularly for
fine garment structures and irregular contours. Precision
and recall further characterize foreground leakage and
missing garment regions, respectively. To evaluate textual
self-refinement, we additionally report IoU improvement
from the initial prediction, improvement and degradation
rates, and inference costs in terms of refinement steps,
language-model calls, and execution time.

\begin{table*}[t]
\centering
\caption{Quantitative comparison of upper-garment segmentation
on DeepFashion2, ModaNet, and Fashionpedia.
The best score for each metric within each dataset
is highlighted in bold. }
\label{tab:garment_segmentation}

\small
\setlength{\tabcolsep}{3.5pt}
\renewcommand{\arraystretch}{0.7}

\begin{tabular*}{\textwidth}{
@{\extracolsep{\fill}}clcccccc@{}
}
\toprule
Dataset & Method & Source
& mIoU$\uparrow$
& Dice$\uparrow$
& Boundary IoU$\uparrow$
& Precision$\uparrow$
& Recall$\uparrow$ \\
\midrule

\multirow{8}{*}{\textbf{DeepFashion2}}
& SEEM
& NeurIPS'23
& 0.4587 & 0.5619 & 0.2372 & 0.4777 & 0.7862 \\
& LISA
& CVPR'24
& 0.6913 & 0.7761 & 0.4688 & 0.7340
& \textbf{0.8828} \\
& SAM 2 (box-only)
& ICLR'25
& 0.6328 & 0.7047 & 0.4124 & 0.7130 & 0.7391 \\
& Grounding DINO + SAM 2
& ECCV'24 / ICLR'25
& 0.2689 & 0.2889 & 0.2122 & 0.3129 & 0.2759 \\
& READ
& CVPR'25
& 0.6067 & 0.6866 & 0.4008 & 0.7728 & 0.6761 \\
& SegAgent
& CVPR'25
& 0.4638 & 0.5675 & 0.2477 & 0.5871 & 0.7178 \\
& SegZero + SAM 2
& arXiv'25 / ICLR'25
& 0.7959 & 0.8629 & 0.6047 & 0.8877 & 0.8765 \\
\cmidrule(l){2-8}
& \textbf{GarmentSegAgent (Ours)}
& --
& \textbf{0.8124} & \textbf{0.8736}
& \textbf{0.6318} & \textbf{0.9156} & 0.8741 \\

\midrule

\multirow{8}{*}{\textbf{ModaNet}}
& SEEM
& NeurIPS'23
& 0.2624 & 0.4009 & 0.1391 & 0.2645
& \textbf{0.9699} \\
& LISA
& CVPR'24
& 0.4538 & 0.5910 & 0.3380 & 0.4634 & 0.9340 \\
& SAM 2 (box-only)
& ICLR'25
& 0.5940 & 0.6765 & 0.5020 & 0.6410 & 0.7590 \\
& Grounding DINO + SAM 2
& ECCV'24 / ICLR'25
& 0.1304 & 0.1397 & 0.1146 & 0.1433 & 0.1388 \\
& READ
& CVPR'25
& 0.5143 & 0.6292 & 0.3906 & 0.5400 & 0.8138 \\
& SegAgent
& CVPR'25
& 0.2996 & 0.4428 & 0.1812 & 0.3122 & 0.9549 \\
& SegZero + SAM 2
& arXiv'25 / ICLR'25
& 0.6683 & 0.7337 & 0.5814 & 0.7105 & 0.8021 \\
\cmidrule(l){2-8}
& \textbf{GarmentSegAgent (Ours)}
& --
& \textbf{0.6733} & \textbf{0.7476}
& \textbf{0.6027} & \textbf{0.7220} & 0.8301 \\

\midrule

\multirow{8}{*}{\textbf{Fashionpedia}}
& SEEM
& NeurIPS'23
& 0.3043 & 0.4376 & 0.1685 & 0.3059 & 0.9078 \\
& LISA
& CVPR'24
& 0.5326 & 0.6545 & 0.4086 & 0.5456 & 0.9281 \\
& SAM 2 (box-only)
& ICLR'25
& 0.6989 & 0.7627 & 0.6180 & 0.7312 & 0.8642 \\
& Grounding DINO + SAM 2
& ECCV'24 / ICLR'25
& 0.2191 & 0.2291 & 0.2012 & 0.2299 & 0.2328 \\
& READ
& CVPR'25
& 0.4496 & 0.5412 & 0.3554 & 0.5089 & 0.6350 \\
& SegAgent
& CVPR'25
& 0.3558 & 0.5034 & 0.2063 & 0.3613 & 0.9158 \\
& SegZero + SAM 2
& arXiv'25 / ICLR'25
& 0.7426 & 0.8158 & 0.6662 & 0.7706 & 0.9332 \\
\cmidrule(l){2-8}
& \textbf{GarmentSegAgent (Ours)}
& --
& \textbf{0.7685} & \textbf{0.8375}
& \textbf{0.6950} & \textbf{0.7940}
& \textbf{0.9410} \\

\bottomrule
\end{tabular*}
\end{table*}

\subsection{Main Results}
\label{sec:main_results}

\begin{table*}[t]
\centering
\caption{Comparison of textual refinement strategies on
three fashion segmentation benchmarks. BIoU denotes
Boundary IoU. All iterative methods share the same
foundation models and maximum refinement budget.}
\label{tab:optimization_comparison}

\small
\setlength{\tabcolsep}{5pt}
\renewcommand{\arraystretch}{0.7}

\begin{tabular*}{\textwidth}{
@{\extracolsep{\fill}}lcccccc@{}
}
\toprule
& \multicolumn{2}{c}{DeepFashion2}
& \multicolumn{2}{c}{ModaNet}
& \multicolumn{2}{c}{Fashionpedia} \\
\cmidrule(lr){2-3}
\cmidrule(lr){4-5}
\cmidrule(lr){6-7}

Method
& mIoU$\uparrow$ & BIoU$\uparrow$
& mIoU$\uparrow$ & BIoU$\uparrow$
& mIoU$\uparrow$ & BIoU$\uparrow$ \\
\midrule

One-pass
& 0.6961 & 0.6148
& 0.6145 & 0.5369
& 0.6982 & 0.6080 \\

\midrule

Rule-based Refinement
& 0.7014 & 0.6179
& 0.6399 & 0.5534
& 0.7080 & 0.6079 \\

Vanilla TextGrad
& 0.7568 & 0.6242
& 0.6001 & 0.5221
& 0.7235 & 0.6304 \\

All-field TextGrad
& 0.7491 & 0.6218
& 0.6283 & 0.5392
& 0.7128 & 0.6153 \\

Selective TextGrad
& 0.7836 & 0.6287
& 0.6454 & 0.5558
& 0.7339 & 0.6612 \\

\midrule

\textbf{Full CAST (Ours)}
& \textbf{0.8124} & \textbf{0.6318}
& \textbf{0.6733} & \textbf{0.6027}
& \textbf{0.7685} & \textbf{0.6950} \\

\bottomrule
\end{tabular*}
\end{table*}

% We compare GarmentSegAgent with representative segmentation
% methods across three fashion benchmarks to assess its
% effectiveness under diverse visual conditions.
% The evaluation considers region overlap, boundary accuracy,
% and foreground prediction quality, with quantitative
% results reported in Table~\ref{tab:garment_segmentation}. 
% Additional experiments examine the transferability of
% the agentic framework to person segmentation and
% the consistency of CAST across different fashion image
% sources (Appendix~\ref{app:additional_results}).

Table~\ref{tab:garment_segmentation} compares
GarmentSegAgent with representative baselines on
DeepFashion2, ModaNet, and Fashionpedia in terms
of region overlap, boundary accuracy, and foreground
prediction quality. Additional experiments evaluate
GarmentSegAgent on COCO~\cite{lin2014coco}
person segmentation and examine CAST performance
across shop and consumer images in DeepFashion2
(Appendix~\ref{app:additional_results}~Table~\ref{tab:app_coco}, Fig.~\ref{fig:app_coco_qualitative}, and Table~\ref{tab:app_domains}).

\noindent\textbf{Performance on DeepFashion2.}
GarmentSegAgent achieves an mIoU of 0.8124,
outperforming the strongest baseline, SegZero+SAM 2,
by 1.65 percentage points. Boundary IoU also improves
from 0.6047 to 0.6318, indicating better preservation
of garment boundaries. Moreover, GarmentSegAgent
achieves the highest precision (0.9156), reflecting
reduced foreground leakage. Although LISA attains
slightly higher recall, its considerably lower precision
and mIoU suggest a less favorable balance between
garment coverage and background exclusion.

% \noindent\textbf{Performance on ModaNet.}
% The performance difference is more modest on ModaNet,
% where GarmentSegAgent improves mIoU from 0.6683 to
% 0.6733 over SegZero+SAM 2. A more pronounced gain
% appears in Boundary IoU, which increases from 0.5814
% to 0.6027, indicating better alignment with garment
% contours in complex street-fashion scenes. Interestingly,
% SEEM and SegAgent achieve considerably higher recall
% but much lower precision, suggesting a tendency to
% include irrelevant foreground regions. GarmentSegAgent
% maintains a more balanced trade-off, achieving a
% precision of 0.7220 and recall of 0.8301.

\noindent\textbf{Performance on ModaNet.}
On ModaNet, the performance advantage is more modest
in region overlap but more evident in boundary accuracy.
Compared with SegZero+SAM 2, GarmentSegAgent improves
mIoU from 0.6683 to 0.6733, while Boundary IoU increases
from 0.5814 to 0.6027, suggesting better delineation
of garment contours in street-fashion scenes.
Although SEEM and SegAgent achieve higher recall,
their substantially lower precision indicates greater
foreground leakage. In contrast, GarmentSegAgent
achieves a precision of 0.7220 and recall of 0.8301,
outperforming SegZero+SAM 2 on both metrics and
maintaining a better balance between garment coverage
and background exclusion.

% \noindent\textbf{Performance on Fashionpedia.}
% Fashionpedia exhibits the most consistent gains across
% evaluation metrics. GarmentSegAgent leads all competing
% methods in mIoU, Dice, Boundary IoU, precision, and
% recall, reaching an mIoU of 0.7685 and a Boundary IoU
% of 0.6950. Relative to SegZero+SAM 2, these represent
% improvements of 2.59 and 2.88 percentage points,
% respectively. The simultaneous gains in precision
% and recall suggest that improved garment coverage
% does not come at the expense of increased foreground
% leakage, supporting more accurate extraction under
% fine-grained garment annotations.

\noindent\textbf{Performance on Fashionpedia.}
Fashionpedia shows the most consistent performance
gains, with GarmentSegAgent outperforming all baselines
across five evaluation metrics. Compared with
SegZero+SAM 2, mIoU and Boundary IoU improve by
2.59 and 2.88 percentage points, reaching 0.7685
and 0.6950, respectively. Notably, precision and
recall increase simultaneously, indicating more
complete garment coverage with reduced foreground
leakage. These results demonstrate improved region
and boundary accuracy in a benchmark featuring
fine-grained garment annotations.

% Across all three benchmarks, GarmentSegAgent achieves
% the highest region overlap and boundary accuracy,
% although the magnitude of improvement varies across
% datasets. We further investigate the role of textual
% optimization and iterative refinement in the
% following sections.

\begin{figure*}[t]
    
    \centering
    \includegraphics[width=0.95\textwidth]
    {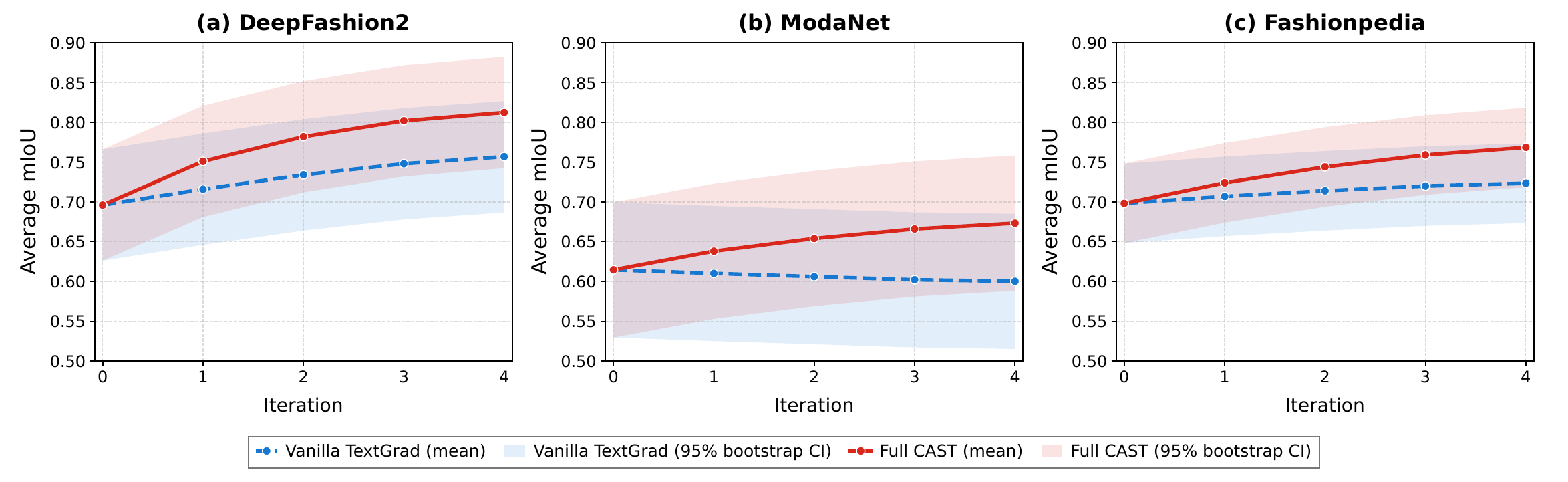}
    % \captionsetup{skip=3pt}
    \caption{Evolution of mIoU during iterative refinement
    on (a) DeepFashion2, (b) ModaNet, and (c) Fashionpedia.
    Iteration 0 denotes the initial prediction. Shaded
    regions indicate 95\% bootstrap confidence intervals.
    Predictions from early-terminated samples are retained
    for subsequent iterations.}
    \label{fig:refinement_dynamics}
\end{figure*}

\begin{table*}[t]
\centering
\caption{Refinement outcomes and inference cost on three
fashion benchmarks. IR and DR denote the percentages
of improved and degraded predictions, respectively.}
\label{tab:refinement_analysis}

\small
\setlength{\tabcolsep}{5pt}
\renewcommand{\arraystretch}{0.75}

\begin{tabular*}{\textwidth}{
@{\extracolsep{\fill}}clcccccc@{}
}
\toprule
& &
\multicolumn{3}{c}{Refinement Outcomes}
& \multicolumn{3}{c}{Inference Cost} \\
\cmidrule(lr){3-5}
\cmidrule(l){6-8}

Dataset & Method
& $\Delta$IoU$\uparrow$
& IR (\%)$\uparrow$
& DR (\%)$\downarrow$
& Avg. Steps$\downarrow$
& LLM Calls$\downarrow$
& Time (s/img)$\downarrow$ \\
\midrule

\multirow{2}{*}{DeepFashion2}
& Vanilla TextGrad
& +0.0607 & 72.4 & 18.0 & 2.2 & 4.4 & 7.8 \\
& \textbf{Full CAST}
& \textbf{+0.1163} & \textbf{88.7}
& \textbf{6.4} & 2.6 & 6.1 & 10.9 \\

\midrule

\multirow{2}{*}{ModaNet}
& Vanilla TextGrad
& -0.0144 & 43.8 & 48.2 & 2.4 & 4.8 & 8.5 \\
& \textbf{Full CAST}
& \textbf{+0.0588} & \textbf{79.6}
& \textbf{12.0} & 2.7 & 6.4 & 11.8 \\

\midrule

\multirow{2}{*}{Fashionpedia}
& Vanilla TextGrad
& +0.0253 & 61.5 & 29.2 & 2.2 & 4.4 & 7.9 \\
& \textbf{Full CAST}
& \textbf{+0.0703} & \textbf{83.5}
& \textbf{9.2} & 2.5 & 5.9 & 10.6 \\

\bottomrule
\end{tabular*}
% \vspace{-0.8em}
\end{table*}

\subsection{Textual Optimization Strategies}
\label{sec:textual_optimization_strategies}

% To examine the role of textual optimization in autonomous
% garment segmentation, we compare five refinement strategies
% against a one-pass reference. \textit{One-pass} retains the
% initial segmentation without refinement, whereas
% \textit{Rule-based Refinement} applies predefined instruction
% updates based on evaluator feedback. \textit{Vanilla TextGrad}
% optimizes the entire instruction as a single textual variable,
% while \textit{All-field TextGrad} updates all four structured
% instruction components independently. \textit{Selective
% TextGrad} restricts optimization to error-relevant components,
% and \textit{Full CAST} further incorporates garment-contract
% validation to maintain instruction consistency.

To investigate the effects of textual refinement,
we compare five optimization strategies against a
\textit{One-pass} reference without refinement.
\textit{Rule-based Refinement} applies predefined
corrections based on evaluator feedback, while
\textit{Vanilla TextGrad} optimizes the entire
instruction as a single variable.
\textit{All-field TextGrad} independently updates
all four instruction fields, whereas
\textit{Selective TextGrad} updates only fields
associated with visual errors.
\textit{Full CAST} additionally incorporates
garment-contract validation. All iterative methods
share the same foundation models, initial instructions,
evaluation protocol, and iteration budget.
Table~\ref{tab:optimization_comparison} reports
the results.

% All iterative strategies use the same foundation models,
% initial instructions, visual evaluation protocol, and maximum
% refinement budget. Their differences lie in how visual feedback
% is translated into textual updates. The comparative results
% on DeepFashion2, ModaNet, and Fashionpedia are reported in
% Table~\ref{tab:optimization_comparison}.

% The results reveal substantial differences among textual
% optimization strategies. Vanilla TextGrad improves segmentation
% on DeepFashion2 and Fashionpedia but reduces mIoU on ModaNet
% from 0.6145 to 0.6001 relative to One-pass, indicating that
% unrestricted instruction refinement can be unreliable.
% Updating all instruction components independently does not
% consistently resolve this issue. In contrast, Selective
% TextGrad outperforms both Vanilla and All-field TextGrad
% across all three benchmarks, suggesting that restricting
% updates to error-relevant variables is beneficial.

The results highlight the limitations of unrestricted
textual refinement. Although Vanilla TextGrad improves
mIoU on DeepFashion2 and Fashionpedia, it degrades
performance on ModaNet from 0.6145 to 0.6001
relative to One-pass. Updating all instruction fields
does not consistently improve performance either.
By contrast, Selective TextGrad outperforms both
variants across all three benchmarks, supporting
the effectiveness of error-conditioned field selection.

% Full CAST achieves the highest mIoU and Boundary IoU on
% every benchmark. Compared with Selective TextGrad, it
% improves mIoU by 2.88, 2.79, and 3.46 percentage points
% on DeepFashion2, ModaNet, and Fashionpedia, respectively.
% The corresponding boundary improvements are particularly
% pronounced on ModaNet and Fashionpedia, indicating more
% accurate garment delineation. These findings support
% the effectiveness of combining selective textual
% optimization with garment-specific consistency control.

Full CAST further achieves the highest mIoU and
Boundary IoU on every benchmark. Relative to Selective
TextGrad, its mIoU gains reach 2.88, 2.79, and
3.46 percentage points on DeepFashion2, ModaNet,
and Fashionpedia, respectively. Boundary IoU improvements
are especially notable on ModaNet and Fashionpedia,
reaching 4.69 and 3.38 percentage points.
These results underscore the complementary roles
of selective instruction updates and garment-specific
consistency validation in reliable textual refinement.

\subsection{Self-Refinement Analysis}
\label{sec:refinement_analysis}

% \noindent\textbf{Iterative Refinement Dynamics.}
% Figure~\ref{fig:refinement_dynamics} tracks the evolution
% of mIoU under Vanilla TextGrad and CAST, starting from
% identical initial predictions. On DeepFashion2 and
% Fashionpedia, both methods improve with successive
% refinement steps, but CAST achieves larger gains and
% maintains a widening performance advantage. A different
% pattern emerges on ModaNet, where Vanilla TextGrad
% gradually decreases from 0.6145 to 0.6001, while CAST
% steadily improves to 0.6733. These contrasting trajectories
% indicate that additional refinement iterations alone do
% not necessarily improve segmentation quality; the way
% visual feedback guides textual updates is equally
% important. For early-terminated samples, the final
% predictions are carried forward to ensure evaluation
% over a consistent image set. 
% A representative refinement trace illustrating how
% visual errors guide selective instruction updates
% is provided in Appendix~\ref{app:cast_analysis}
% (Table~\ref{tab:app_prompt_evolution}).

\noindent\textbf{Iterative Refinement Dynamics.}
Figure~\ref{fig:refinement_dynamics} compares the evolution
of mIoU across refinement iterations for Vanilla TextGrad
and CAST under identical initial predictions. On DeepFashion2
and Fashionpedia, both methods improve through iterative
refinement, with CAST achieving larger gains.
A contrasting trend appears on ModaNet, where Vanilla
TextGrad reduces mIoU from 0.6145 to 0.6001, whereas
CAST progressively improves it to 0.6733. These divergent
trajectories demonstrate that repeated refinement does
not inherently improve segmentation; its effectiveness
depends on how visual errors guide textual updates.
For early-terminated samples, the final predictions
are retained in subsequent plotted iterations to
maintain a consistent evaluation set.
% A trace of error-attributed instruction
% updates is provided in Appendix~\ref{app:cast_analysis}
% (Table~\ref{tab:app_prompt_evolution}).
Further examples of iterative mask evolution and instruction updates are provided
in Appendix~\ref{app:additional_results}
(Fig.~\ref{fig:iteration_refinement}) and
Appendix~\ref{app:cast_analysis}
(Table~\ref{tab:app_prompt_evolution}), respectively. 
An illustrative case study of error-attributed
instruction refinement is provided in
Appendix~\ref{app:prompt_evolution}.

% \begin{figure*}[t]
%     \centering
%     \includegraphics[width=0.96\textwidth]
%     {pictures/iteration_performance_legend_updated_lower.pdf}
%     \caption{
%     Iterative refinement performance on three fashion
%     segmentation benchmarks: (a) DeepFashion2,
%     (b) ModaNet, and (c) Fashionpedia.
%     We compare the evolution of mIoU between conventional
%     TextGrad and the proposed CAST across successive
%     segmentation iterations. For samples terminated early,
%     their final predictions are retained in subsequent
%     evaluation steps.
%     }
%     \label{fig:refinement_dynamics}
% \end{figure*}

% \noindent\textbf{Refinement Reliability.}
% Average segmentation gains do not fully capture the risk
% of harmful instruction updates. We therefore measure
% the mean IoU change between initial and final predictions
% ($\Delta$IoU), together with the improvement rate (IR)
% and degradation rate (DR), defined as the proportions
% of samples whose IoU increases or decreases after
% refinement. As shown in Table~\ref{tab:refinement_analysis},
% CAST achieves larger refinement gains and substantially
% lower degradation rates across all benchmarks.
% The difference is particularly pronounced on ModaNet,
% where Vanilla TextGrad degrades 48.2\% of predictions,
% compared with 12.0\% for CAST. Similar reductions
% are observed on DeepFashion2 and Fashionpedia, indicating
% more reliable image-specific refinement.

\begin{figure*}[]
    \centering
    \includegraphics[width=0.9\textwidth]
    {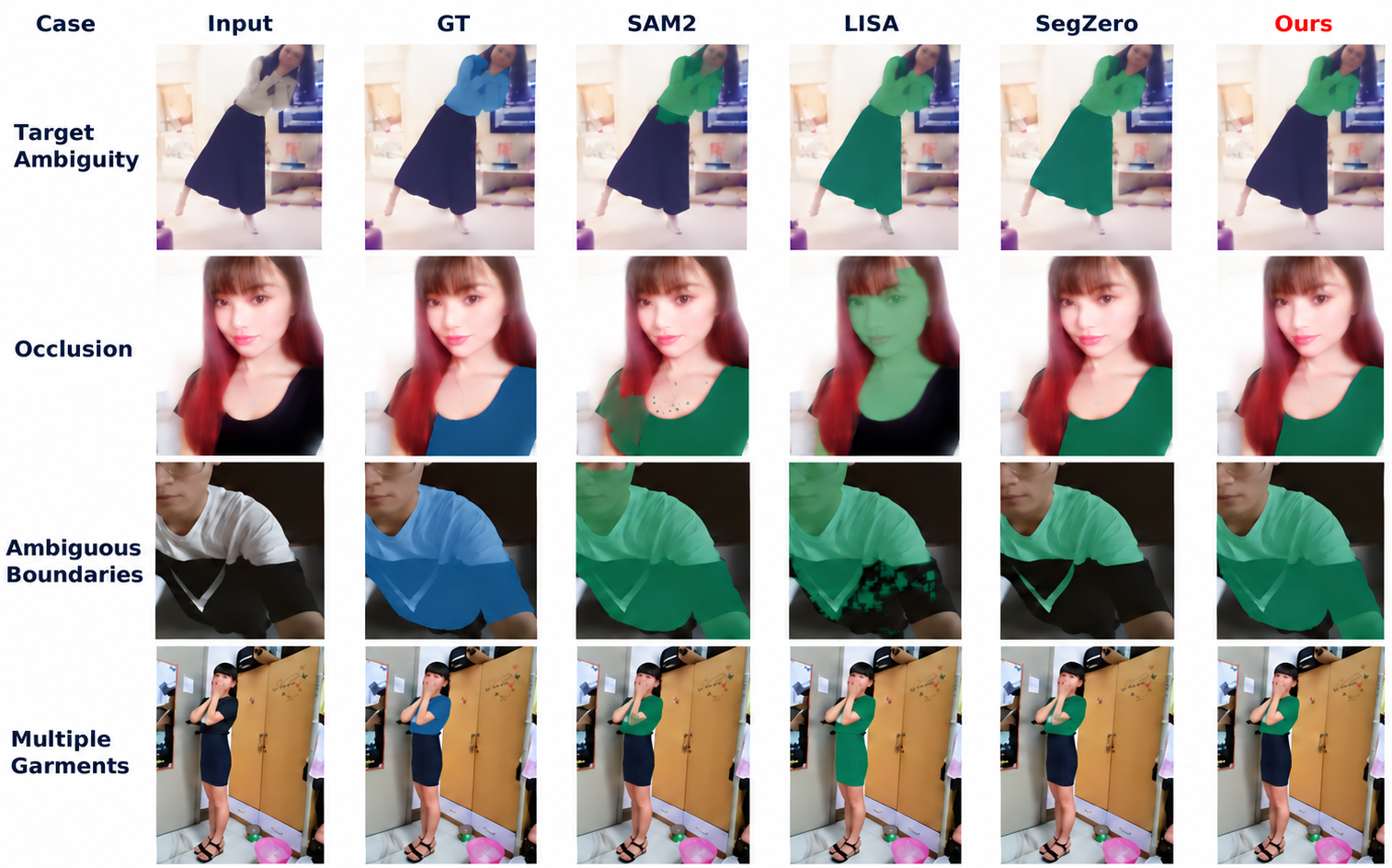}
    % \caption{
    % Qualitative analysis of CAST-guided garment segmentation
    % refinement. Representative examples illustrate
    % evaluator-generated visual feedback, selected textual
    % instruction components, and the resulting segmentation
    % updates. Ground-truth masks are provided only for
    % offline comparison and are not accessible during inference.
    % }
    % \caption{Qualitative comparison of upper-garment
    % segmentation under challenging visual conditions,
    % including target ambiguity, occlusion, ambiguous
    % boundaries, and multiple garments (top to bottom).
    % Columns show the input image, ground truth (GT),
    % and predictions from SAM 2, LISA, SegZero + SAM 2,
    % and GarmentSegAgent (Ours). Ground truth is used
    % only for offline evaluation.}

\caption{Qualitative comparison of upper-garment segmentation under target ambiguity, occlusion, ambiguous boundaries, and multiple garments (top to bottom). Columns show the input, ground truth (GT), and predictions from SAM 2, LISA, SegZero + SAM 2, and GarmentSegAgent (Ours). GT is used only for evaluation.}
    \label{fig:qualitative_cast}
    \vspace{-0.8em}
\end{figure*}

\noindent\textbf{Refinement Reliability.}
Average segmentation gains alone do not capture the
risk of harmful instruction updates. We therefore
report the mean IoU change ($\Delta$IoU) between initial
and final predictions, along with improvement rate (IR)
and degradation rate (DR), measuring the proportions
of samples with increased and decreased IoU, respectively.
As shown in Table~\ref{tab:refinement_analysis},
CAST consistently achieves larger IoU gains, higher
improvement rates, and lower degradation rates.
The contrast is particularly evident on ModaNet,
where CAST increases IR from 43.8\% to 79.6\% while
reducing DR from 48.2\% to 12.0\%. Similar reductions
in degradation are observed on DeepFashion2 and
Fashionpedia, indicating more reliable image-specific
textual refinement.

% \noindent\textbf{Inference Cost.}
% The reliability gains come with additional computational
% overhead. CAST requires more refinement steps and
% language-model calls than Vanilla TextGrad, increasing
% average inference time by approximately three seconds
% per image across the evaluated benchmarks. This trade-off
% reflects the additional processing associated with
% error-attributed optimization and garment-constrained
% validation, rather than a reduction in inference cost.

\noindent\textbf{Inference Cost.}
These reliability improvements incur additional
inference overhead. Compared with Vanilla TextGrad,
CAST requires approximately 1.5--1.7 more language-model
calls per image, with average inference time increasing
by 2.7--3.3 seconds across the three benchmarks.
For instance, on ModaNet, inference time increases
from 8.5 to 11.8 seconds per image, accompanied by
a substantial reduction in degradation rate.
This trade-off reflects the additional computation
associated with error-attributed refinement and
garment-constrained validation, rather than improved
inference efficiency.

\subsection{Ablation Studies}
\label{sec:ablation_studies}

We investigate the contributions of three core CAST
mechanisms: visual error attribution (VEA), selective
routing, and garment-contract validation.
As shown in Table~\ref{tab:cast_ablation},
removing VEA reduces mIoU from 0.8124 to 0.7128
and Boundary IoU from 0.6318 to 0.4924,
highlighting the importance of structured visual
feedback. Disabling selective routing similarly
degrades mIoU to 0.7287 and Boundary IoU to 0.5121,
supporting the benefit of error-guided instruction
updates. Without garment-contract validation,
mIoU drops to 0.7128, while the degradation rate
increases from 6.4\% to 18.0\%, indicating its
role in mitigating harmful updates.
These findings demonstrate the complementary
contributions of error attribution, selective
optimization, and constraint validation.
Additional validation statistics are provided
in Appendix~\ref{app:cast_analysis}
(Table~\ref{tab:app_validation}).

\begin{table}[t]
\centering
% \caption{Ablation of CAST mechanisms on DeepFashion2.
% DR denotes the percentage of samples with lower final
% than initial IoU.}

\caption{Ablation of CAST on DeepFashion2. DR: percentage of samples with decreased final IoU.}

\label{tab:cast_ablation}

\small
\setlength{\tabcolsep}{15pt}
\renewcommand{\arraystretch}{0.7}

\begin{tabular}{@{}lccc@{}}
\toprule
Variant
& mIoU$\uparrow$
& BIoU$\uparrow$
& DR (\%)$\downarrow$ \\
\midrule

\textbf{Full CAST}
& \textbf{0.8124}
& \textbf{0.6318}
& \textbf{6.4} \\

w/o VEA
& 0.7128 & 0.4924 & 14.0 \\

w/o Selective Routing
& 0.7287 & 0.5121 & 14.0 \\

w/o Garment Validation
& 0.7128 & 0.4917 & 18.0 \\

\bottomrule
\end{tabular}
\vspace{-1.15em}
\end{table}

% \subsection{Qualitative Analysis}

% We further provide qualitative comparisons to analyze
% the behavior of SegAgent.

% As shown in Fig.~\ref{fig:qualitative},
% existing approaches often fail under challenging conditions,including:

% \begin{itemize}

% \item similar appearance between garments and backgrounds;

% \item complex human poses;

% \item severe occlusion;

% \item multiple clothing instances.

% \end{itemize}

% Through evaluator-guided refinement,
% SegAgent progressively removes irrelevant regions
% and restores missing garment boundaries.

% These qualitative results demonstrate that SegAgent performs
% self-correcting segmentation rather than one-pass prediction.

% \begin{figure}[t]

% \centering

% \includegraphics[width=\linewidth]
% {pictures/result_quality.png}

% \caption{
% Qualitative comparison of garment segmentation results.
% }

% \label{fig:qualitative}

% \end{figure}

% \subsection{Failure Analysis}

% Although SegAgent improves robustness under complex scenarios,
% several limitations remain.

% Typical failure cases include:

% \begin{itemize}

% \item
% severe garment occlusion where visual evidence is insufficient;

% \item
% ambiguous boundaries between garments and human regions;

% \item
% extremely low-resolution input images;

% \item
% multiple overlapping garments with similar appearance.

% \end{itemize}

% These limitations indicate several future research directions:

% \begin{itemize}

% \item stronger garment-specific visual reasoning;

% \item more reliable evaluator feedback mechanisms;

% \item improved long-horizon agent planning.

% \end{itemize}

\subsection{Qualitative Comparison}
\label{sec:qualitative_analysis}

% We further examine segmentation quality under four
% challenging scenarios: target ambiguity, occlusion,
% ambiguous garment boundaries, and multiple garments.
% As illustrated in Fig.~\ref{fig:qualitative_cast},
% GarmentSegAgent produces masks that more closely
% match the target garments, while competing methods
% occasionally include unrelated body regions or
% lower-body clothing. The differences are particularly
% evident in cases involving target ambiguity and
% multiple garments, highlighting the importance of
% target-specific localization and region exclusion.

% Qualitative results further illustrate the effectiveness
% of GarmentSegAgent under challenging conditions,
% including target ambiguity, occlusion, ambiguous
% garment boundaries, and multiple garments.
% As shown in Fig.~\ref{fig:qualitative_cast},
% GarmentSegAgent more accurately isolates the target
% garment, preserving visible garment regions while
% reducing leakage into adjacent clothing and background
% areas. In contrast, competing methods may omit
% garment parts or include unrelated regions.
% These differences are particularly evident in scenes
% with ambiguous targets or multiple garments,
% highlighting the importance of precise target
% localization and region exclusion.

Qualitative results demonstrate the effectiveness of GarmentSegAgent under target ambiguity, occlusion, ambiguous boundaries, and multiple garments. As shown in Fig.~\ref{fig:qualitative_cast}, GarmentSegAgent better preserves target garment regions while reducing leakage into adjacent clothing and backgrounds. In contrast, competing methods often omit garment parts or include unrelated regions, highlighting the importance of accurate target localization and region exclusion.

\section{Limitations}
\label{sec:limitations}

% Our current evaluation focuses on single-target
% upper-garment segmentation in Web fashion images.
% Extending the framework to other garment categories
% and simultaneous multi-instance segmentation requires
% further investigation. Moreover, CAST relies on visual
% feedback and candidate masks generated by frozen
% foundation models, while garment-constrained validation
% does not guarantee pixel-level accuracy. Finally,
% iterative evaluation and textual optimization introduce
% additional inference overhead. Future work will
% explore broader garment understanding and more
% efficient refinement strategies.

Our evaluation focuses on single-target upper-garment segmentation in Web fashion images. Extending GarmentSegAgent to broader garment categories and multi-instance scenarios remains unexplored. Moreover, CAST relies on visual feedback and candidate masks from frozen foundation models, while garment-constrained validation cannot guarantee pixel-level accuracy. Iterative refinement also introduces additional inference overhead. Future work will investigate broader garment understanding and more efficient refinement strategies.

\section{Conclusion}
\label{sec:conclusion}

% In this paper, we introduced \textbf{GarmentSegAgent},
% an autonomous framework for garment segmentation in
% heterogeneous Web fashion images. The framework integrates
% semantic planning, spatial execution, and visual evaluation
% to enable iterative refinement without human-provided
% spatial prompts or model fine-tuning. At its core,
% \textbf{CAST} combines visual error attribution,
% selective textual optimization, and garment-constrained
% validation to improve segmentation decisions.
% Experiments on DeepFashion2, ModaNet, and Fashionpedia
% demonstrate consistent improvements in region overlap
% and boundary accuracy, while further analyses show
% more reliable refinement than conventional TextGrad.
% These findings highlight the potential of
% error-guided textual optimization for coordinating
% frozen foundation models in autonomous visual tasks.

We presented GarmentSegAgent, an autonomous framework for garment segmentation in heterogeneous Web fashion images. Its core mechanism, CAST, combines visual error attribution, selective textual optimization, and garment-constrained validation to iteratively improve segmentation without human-provided spatial prompts or fine-tuning. Experiments on DeepFashion2, ModaNet, and Fashionpedia demonstrate consistent improvements in region overlap and boundary
accuracy over competitive baselines. These findings highlight the potential of error-
guided textual optimization for coordinating frozen foundation models in visual tasks.

% \newpage

\bibliographystyle{ACM-Reference-Format}
\bibliography{references}

\appendix

\section{Target Selection and Evaluation Protocol}
\label{app:implementation}

For images containing multiple eligible upper-body
garments, we designate the visible instance with the
largest garment area as the segmentation target.
This selection criterion establishes a consistent
target definition across evaluated methods, particularly
in scenes containing overlapping or adjacent garments.
Predicted masks are evaluated against the corresponding
target-instance annotation rather than other visible
garment instances. Importantly, ground-truth masks and
instance annotations are not accessible to GarmentSegAgent
during planning, segmentation, evaluation, or textual
refinement. They are used exclusively for offline
annotation matching and metric computation, ensuring
that visual feedback remains independent of reference
segmentation masks.

\begin{figure*}[t]
    \centering
    \includegraphics[width=\textwidth]
    {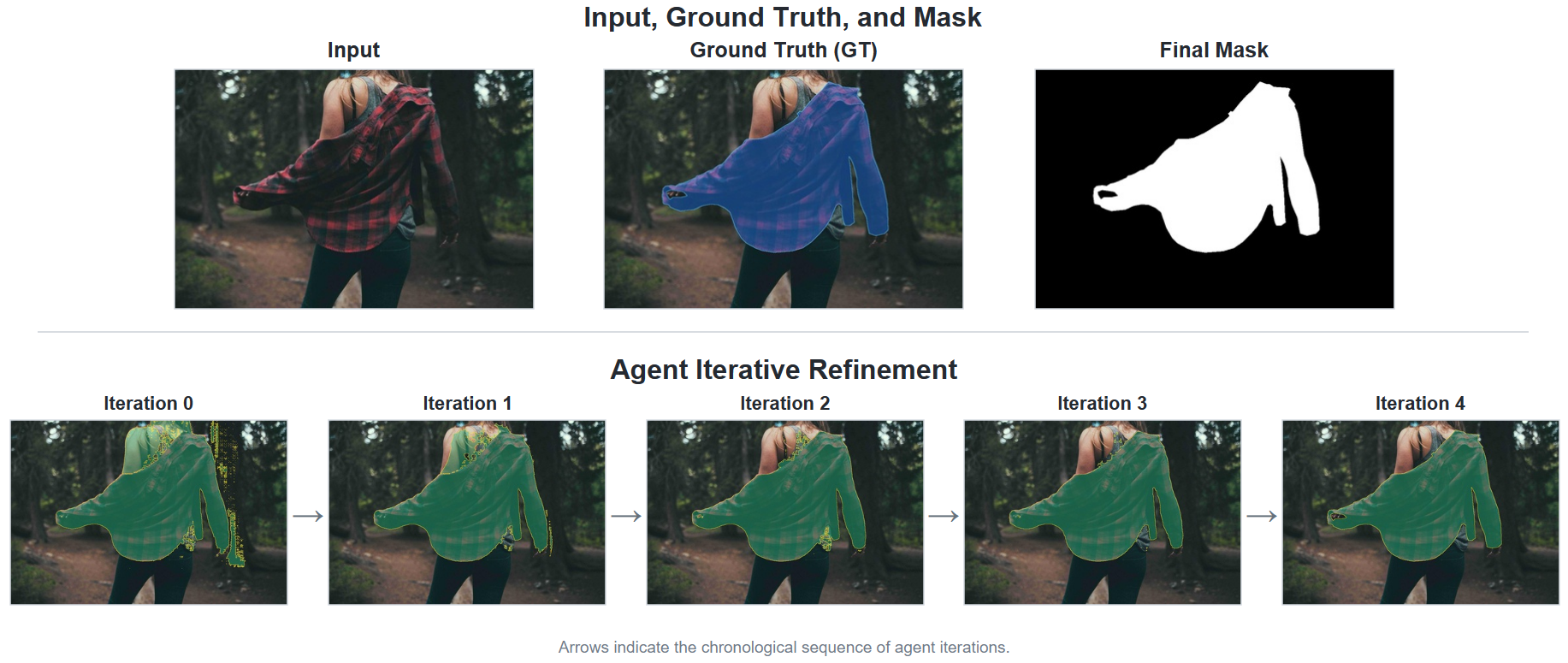}
    \caption{Visualization of iterative garment
    segmentation refinement. The top row presents
    the input image, ground truth (GT), and final
    prediction, while the bottom row shows the
    predicted masks across Iterations 0--4.
    GT is used only for offline comparison.}
    \label{fig:iteration_refinement}
\end{figure*}

\begin{figure*}[]
    \centering
    \includegraphics[width=\textwidth]{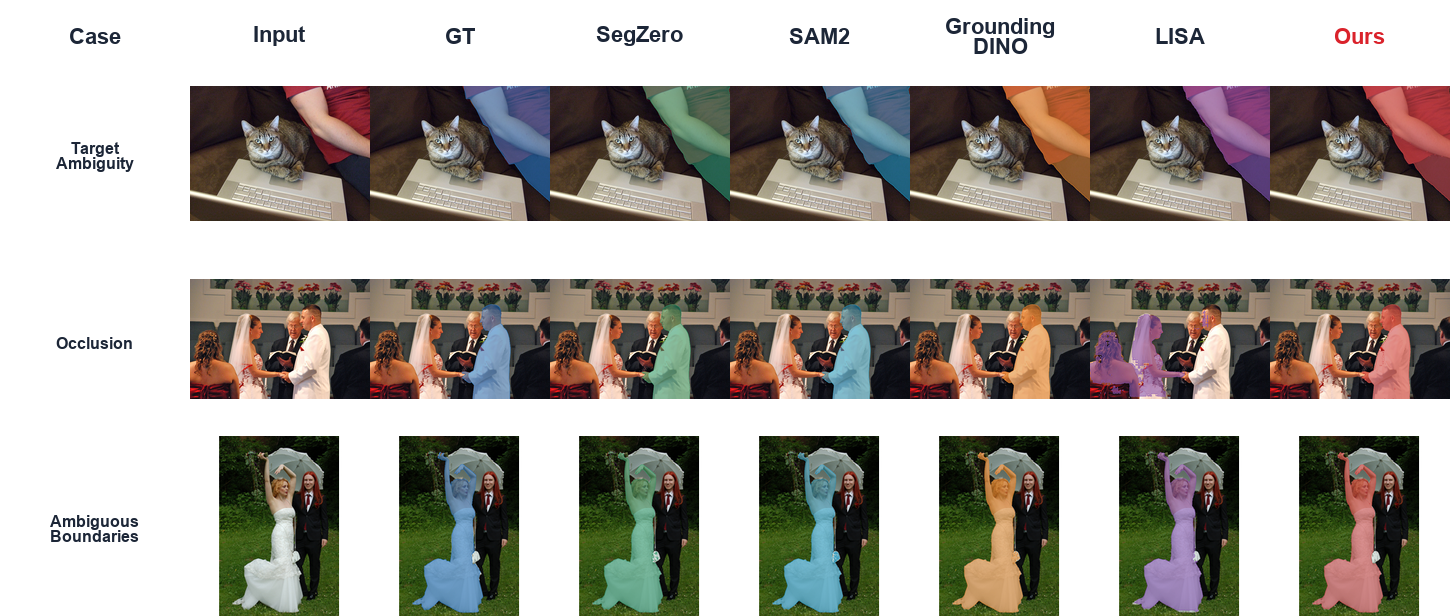}
    
\caption{Qualitative comparison of person segmentation on COCO
under target ambiguity, occlusion, and ambiguous boundaries.}

    \label{fig:app_coco_qualitative}
\end{figure*}

\section{Additional Experimental Results}
\label{app:additional_results}

\noindent\textbf{Visualization of Iterative Refinement.}
To further illustrate the refinement process,
Fig.~\ref{fig:iteration_refinement} presents the
evolution of garment segmentation across successive
iterations. Starting from the initial prediction,
GarmentSegAgent progressively improves garment
coverage and boundary alignment, particularly
around the sleeves and lower edges. This example
provides a visual illustration of how feedback-driven
textual refinement can correct localized segmentation
errors without modifying foundation-model parameters.

\noindent\textbf{Cross-domain Person Segmentation on COCO.}
We further evaluate the generalization of the agentic
segmentation framework on COCO~\cite{lin2014coco} person segmentation,
which differs from garment extraction in both target
semantics and image distribution.
As shown in Table~\ref{tab:app_coco}, the TextGrad-based
GarmentSegAgent achieves an mIoU of 0.7138, outperforming
SegZero by 1.68 percentage points. Improvements are also
observed in Dice, Boundary IoU, and precision, while the
highest recall of 0.9045 indicates stronger foreground
coverage. Grounding DINO + SAM 2 achieves higher region
and boundary accuracy, potentially benefiting from
detector-guided localization of the predefined person
category. In contrast, GarmentSegAgent relies on
language-guided planning and general textual refinement,
without task-specific adaptation for person segmentation.

\begin{table}[!h]
\centering
\caption{Cross-domain person segmentation on COCO.
BIoU denotes Boundary IoU.}
\label{tab:app_coco}
\small
\setlength{\tabcolsep}{3pt}
\renewcommand{\arraystretch}{1.08}
\begin{tabular}{lccccc}
\toprule
Method & mIoU & Dice & BIoU & Prec. & Rec. \\
\midrule
READ
& 0.5820 & 0.6720 & 0.4660 & 0.6460 & 0.6800 \\
SEEM
& 0.6110 & 0.6560 & 0.4990 & 0.6610 & 0.7860 \\
LISA
& 0.6420 & 0.7140 & 0.5250 & 0.7060 & 0.8830 \\
SegZero
& 0.6970 & 0.7520 & 0.5770 & 0.7610 & 0.8770 \\
Grounding DINO+SAM 2
& \textbf{0.7650} & \textbf{0.8210}
& \textbf{0.6450} & \textbf{0.8320} & 0.8500 \\
\midrule
GarmentSegAgent 
& 0.7138 & 0.7826 & 0.6049
& 0.8017 & \textbf{0.9045} \\
\bottomrule
\end{tabular}
\end{table}

Qualitative results in Fig.~\ref{fig:app_coco_qualitative}
further illustrate the cross-domain behavior of GarmentSegAgent
under target ambiguity, occlusion, and ambiguous boundaries.
Despite being developed for garment segmentation, GarmentSegAgent
can identify and segment the target person across diverse scenes,
preserving visible foreground regions while limiting interference
from surrounding objects and individuals. These observations
suggest that its language-guided planning and textual refinement
can transfer beyond garment-specific segmentation tasks.

% \begin{figure*}[!t]
%     \centering
%     \includegraphics[
%         width=\textwidth,
%         keepaspectratio
%     ]{pictures/iteration5_sample_style_from_svg.png}
%     \caption{Visualization of the input image, ground truth mask,
%     final mask, and the iterative refinement process across five agent
%     iterations. The arrows indicate the chronological sequence of the
%     refinement process.}
%     \label{fig:iteration_refinement}
% \end{figure*}

\noindent\textbf{Performance across Fashion Image Sources.}
We further compare segmentation performance between
shop and consumer images in DeepFashion2.
As reported in Table~\ref{tab:app_domains},
GarmentSegAgent with CAST consistently outperforms
its conventional TextGrad counterpart and SegZero
+ SAM 2 across both domains. CAST improves mIoU
over conventional TextGrad by 0.9 and 0.7 percentage
points on shop and consumer images, respectively.
All methods achieve lower accuracy on consumer images,
indicating greater segmentation difficulty in this domain.

% \begin{table}[t]
% \centering
% \caption{Segmentation performance across shop and
% consumer images in DeepFashion2.}
% \label{tab:app_domains}
% \small
% \setlength{\tabcolsep}{5pt}
% \renewcommand{\arraystretch}{1.08}
% \begin{tabular}{llccc}
% \toprule
% Domain & Method & Images & mIoU & BIoU \\
% \midrule
% \multirow{3}{*}{Shop}
% & SegZero+SAM 2 &500 &0.808&0.625 \\
% & TextGrad &500 &0.817 &0.638 \\
% & CAST &500 &0.826 &0.650 \\
% \midrule
% \multirow{3}{*}{Consumer}
% & SegZero+SAM 2 &500 &0.784 &0.584 \\
% & TextGrad &500 &0.792 &0.599 \\
% & CAST &500 &0.799 &0.614 \\
% \bottomrule
% \end{tabular}
% \end{table}

\begin{table}[!h]
\centering
\caption{Segmentation performance across shop and
consumer images in DeepFashion2 (500 images per
domain). GSA denotes GarmentSegAgent, and BIoU
denotes Boundary IoU.}
\label{tab:app_domains}

\small
\setlength{\tabcolsep}{4pt}
\renewcommand{\arraystretch}{1.08}

\begin{tabular*}{\columnwidth}{
@{\extracolsep{\fill}}llcc@{}
}
\toprule
Domain & Method
& mIoU$\uparrow$ & BIoU$\uparrow$ \\
\midrule

\multirow{3}{*}{Shop}
& SegZero + SAM 2
& 0.808 & 0.625 \\
& GSA (TextGrad)
& 0.817 & 0.638 \\
& \textbf{GSA (CAST)}
& \textbf{0.826} & \textbf{0.650} \\

\midrule

\multirow{3}{*}{Consumer}
& SegZero + SAM 2
& 0.784 & 0.584 \\
& GSA (TextGrad)
& 0.792 & 0.599 \\
& \textbf{GSA (CAST)}
& \textbf{0.799} & \textbf{0.614} \\

\bottomrule
\end{tabular*}
\end{table}

\section{CAST Mechanism Analysis}
\label{app:cast_analysis}

\noindent\textbf{Prompt Evolution Case Study.}
We examine the textual refinement process by tracing
how CAST transforms visual segmentation feedback into
targeted instruction updates. As illustrated in
Table~\ref{tab:app_prompt_evolution}, the trace records
the attributed visual error, selected instruction
fields, textual gradient feedback, and instructions
before and after optimization. We also examine whether
unselected fields remain unchanged and how the proposed
updates are handled by garment-contract validation.
The corresponding IoU change is measured offline to
assess the effect of refinement on segmentation quality.

\begin{table}[!h]
\centering
\caption{Condensed CAST refinement trace for Fashionpedia
image 10324. Textual entries summarize the recorded
feedback and updates; IoU is measured offline.}
\label{tab:app_prompt_evolution}

\small
\setlength{\tabcolsep}{3pt}
\renewcommand{\arraystretch}{1.10}

\begin{tabular}{@{}p{0.25\columnwidth}
                   p{0.69\columnwidth}@{}}
\toprule
Component & Refinement Trace \\
\midrule

Visual Errors
& Missing jacket regions (torso, sleeves, collar,
and hem), with erroneous inclusion of hands,
skin, and a fur collar. \\

Selected Fields
& Coverage and exclusion. \\

Initial Instruction
& Preserve visible garment fabric and exclude
body parts, other garments, and background. \\

Textual Feedback
& Re-examine incomplete jacket coverage and
the inclusion of non-garment regions. \\

Updated Fields
& \textit{Coverage}: retain the visible torso,
sleeves, cuffs, collar, and hem.
\textit{Exclusion}: remove body parts,
accessories, unrelated garments, and background. \\

Preserved Rule
& Select the eligible upper-body garment
with the largest visible fabric area. \\

IoU   \hfill \\ (Before $\rightarrow$ After)
& $0.2247 \rightarrow 0.8925$
($\Delta=+0.6678$). \\

\bottomrule
\end{tabular}
\end{table}

\noindent\textbf{Garment-Contract Validation.}
We further examine the outcomes of garment-contract
validation across the three benchmarks.
As reported in Table~\ref{tab:app_validation},
26.0--31.8\% of proposed updates require repair,
while no rollbacks are recorded. Independent
assessment identifies residual constraint-violation
rates of 0.2--0.8\%, suggesting that validation
substantially supports instruction consistency
without ensuring perfect constraint compliance.

\begin{table}[!h]
\centering
\caption{Garment-contract validation outcomes.
Direct, Repair, and Rollback denote the final
handling of proposed updates; Viol. denotes
independently assessed constraint violations.}
\label{tab:app_validation}

\small
\setlength{\tabcolsep}{2pt}
\renewcommand{\arraystretch}{1.10}

\begin{tabular*}{\columnwidth}{
@{\extracolsep{\fill}}lccccc@{}
}
\toprule
Dataset
& Updates
& Direct (\%)
& Repair (\%)
& Rollback (\%)
& Viol. (\%) \\
\midrule

DeepFashion2
& 650 & 74.0 & 26.0 & 0.0 & 0.2 \\

ModaNet
& 640 & 71.9 & 28.1 & 0.0 & 0.3 \\

Fashionpedia
& 530 & 68.2 & 31.8 & 0.0 & 0.8 \\

\bottomrule
\end{tabular*}
\end{table}

\section{Case Study of Iterative Prompt Evolution}
\label{app:prompt_evolution}

We illustrate iterative instruction refinement using
Fashionpedia image 18374, featuring a red-and-black
plaid overshirt. The case covers five segmentation
iterations, from initial garment localization to
boundary refinement. As the original textual logs
were not retained, the following prompts are
illustrative reconstructions based on the available
segmentation masks and CAST protocol, rather than
verbatim execution records.

\begin{iterprompt}{Prompt 0: Initial Garment Specification}
\small
\textbf{Instruction:}
Identify and segment the complete red-and-black plaid
overshirt, including its collar, sleeves, cuffs,
and lower hem. Exclude exposed skin, trousers,
surrounding objects, and background regions.
Preserve the garment's identity and visible structure.
\end{iterprompt}

\begin{iterprompt}{Prompt 1: Background Exclusion}
\small
\textbf{Feedback:} Background leakage and fragmented
garment regions in the initial prediction.

\textbf{Selected fields:} Localization, Coverage, Exclusion.

\textbf{Update:} Focus on the connected plaid fabric
regions forming the overshirt. Recover missing garment
areas while excluding unrelated background pixels.
Preserve the collar, sleeves, and lower hem without
changing the target garment.
\end{iterprompt}

\begin{iterprompt}{Prompt 2: Contour Refinement}
\small
\textbf{Feedback:} Residual background leakage and
overly broad spatial localization near garment edges.

\textbf{Selected fields:} Localization, Exclusion.

\textbf{Update:} Tighten the upper and right garment
boundaries to follow the visible fabric contours.
Exclude adjacent skin and tree background regions
while preserving the natural sleeve--torso gap
and previously identified garment structures.
\end{iterprompt}

\begin{iterprompt}{Prompt 3: Local Region Recovery}
\small
\textbf{Feedback:} Missing regions near the cuff and
lower-right hem, with residual leakage along the
left sleeve.

\textbf{Selected fields:} Localization, Coverage, Exclusion.

\textbf{Update:} Recover omitted plaid fabric near
the cuff and lower hem. Refine the left sleeve contour
to suppress surrounding background while preserving
thin garment regions and natural boundaries.
Avoid unnecessary expansion into adjacent areas.
\end{iterprompt}

\begin{iterprompt}{Prompt 4: Final Boundary Refinement}
\small
\textbf{Feedback:} Minor residual background near
garment boundaries, with most visible regions recovered.

\textbf{Selected fields:} Localization, Exclusion.

\textbf{Update:} Apply localized corrections to remove
remaining background leakage. Maintain alignment with
the visible plaid fabric while preserving the collar,
sleeve contours, cuffs, and lower hem. Avoid altering
correctly segmented garment regions.
\end{iterprompt}

\paragraph{Analysis.}
This illustrative sequence demonstrates how different
visual errors can activate distinct instruction fields
during iterative refinement. Early updates address
background leakage and incomplete garment coverage,
while later refinements focus on localized boundaries
and residual errors. Throughout the sequence, the
target instruction remains unchanged, whereas
localization, coverage, and exclusion are selectively
revised according to the attributed errors.
Garment-contract validation checks candidate instructions
for task consistency before subsequent execution.
Although the reconstructed prompts are not verified
execution logs, they illustrate the intended interaction
between visual error attribution, selective TextGrad,
and constrained instruction updates.

\section{Theoretical Analysis of CAST}
\label{app:theoretical_analysis}

We analyze two structural properties of CAST:
the locality of selective textual updates and the
preservation of garment-specific constraints.

\noindent\textbf{Selective Update Locality.}
Let $p_t=(p_t^{(1)},\ldots,p_t^{(m)})$ denote the
structured instruction at iteration $t$, where
$m=4$ corresponds to the target, localization,
coverage, and exclusion fields.
Given the error-attributed active set
$\mathcal{K}_t\subseteq\{1,\ldots,m\}$,
CAST generates a candidate update as
\begin{equation}
\widetilde{p}_{t+1}^{(j)}=
\begin{cases}
\mathcal{U}_{\mathrm{TGD}}
(p_t^{(j)},\gamma_t^{(j)}),
& j\in\mathcal{K}_t,\\
p_t^{(j)},
& j\notin\mathcal{K}_t.
\end{cases}
\end{equation}
The final instruction is obtained through
$p_{t+1}=\mathcal{G}_{g,\mathcal{K}_t}
(\widetilde{p}_{t+1},p_t)$.

\begin{proposition}[Update Locality]
\label{prop:update_locality}
Suppose the garment-constrained update operator
preserves all inactive fields, including during
validation and repair. Then
\begin{equation}
p_{t+1}^{(j)}=p_t^{(j)},
\quad \forall j\notin\mathcal{K}_t,
\end{equation}
and the number of modified instruction fields
satisfies
\begin{equation}
\left|
\left\{j:p_{t+1}^{(j)}\neq p_t^{(j)}\right\}
\right|
\leq |\mathcal{K}_t|\leq m.
\end{equation}
\end{proposition}

\begin{proof}
By construction, selective TextGrad leaves every
inactive field unchanged. The assumed field-preserving
validation operator does not modify these fields.
Consequently, only fields belonging to
$\mathcal{K}_t$ can change, establishing the bound.
\end{proof}

\noindent\textbf{Garment-Contract Feasibility.}
Let $\mathcal{C}(g)$ denote the set of instructions
satisfying the garment contract $g$.
CAST validates each candidate instruction and
accepts a direct or repaired update only when it
passes the contract check; otherwise, it restores
the previous instruction.

\begin{proposition}[Constraint Preservation]
\label{prop:constraint_preservation}
Assume (i) $p_0\in\mathcal{C}(g)$,
(ii) the contract validator is sound, such that
every accepted instruction belongs to
$\mathcal{C}(g)$, and (iii) unsuccessful updates
are rolled back to $p_t$.
Then
\begin{equation}
p_t\in\mathcal{C}(g),
\quad \forall t\in\{0,\ldots,T\}.
\end{equation}
\end{proposition}

\begin{proof}
The claim holds at $t=0$ by assumption.
Suppose $p_t\in\mathcal{C}(g)$.
If the candidate or its repaired version passes
sound validation, the accepted instruction belongs
to $\mathcal{C}(g)$. Otherwise, rollback gives
$p_{t+1}=p_t\in\mathcal{C}(g)$.
The result follows by induction.
\end{proof}

\noindent\textbf{Scope of the Guarantees.}
These propositions characterize instruction-level
properties rather than segmentation accuracy.
Selective updates restrict which textual variables
may change, but do not bound the resulting mask
variation. Similarly, contract preservation does
not guarantee that a revised instruction improves
IoU. In practice, language-model-based validation
may also be imperfect, making the feasibility
guarantee conditional on validator soundness.
We therefore complement these properties with
empirical analyses of refinement gains, degradation
rates, and validation outcomes.

\end{document}